\documentclass[journal,twoside,web]{ieeecolor}
\usepackage{cite}
\makeatletter
\let\NAT@parse\undefined
\makeatother
\usepackage{generic}
\usepackage{amsmath,amssymb,amsfonts}
\usepackage{algorithm}
\usepackage{graphicx}
\usepackage{textcomp}
\usepackage{booktabs}
\usepackage{array}
\usepackage{flushend}
\usepackage{threeparttable}
\usepackage{tabularx}
\usepackage{bm}
\usepackage{multirow}
\usepackage{makecell}
\usepackage{soul, color, xcolor}
\usepackage{algpseudocode}
\usepackage{hyperref}
\hypersetup{colorlinks = true,
			linkcolor = blue,
			citecolor = blue,
			urlcolor = blue}

\newtheorem{definition}{Definition}

\newtheorem{corollary}{Corollary}

\def\BibTeX{{\rm B\kern-.05em{\sc i\kern-.025em b}\kern-.08em
		T\kern-.1667em\lower.7ex\hbox{E}\kern-.125emX}}
\begin{document}
\title{Interaction-Stiffness-Guided Basis Allocation \\in Dynamic Movement Primitives for \\Efficient Skill Transfer}

\author{
	Chan Xu,
	Silu Chen,~\emph{Senior Member,~IEEE},
	Dehao Wang,
	Xiyu Chen,
	Dexin Jiang,\\ 
	Chi~Zhang,~\emph{Senior Member, IEEE},
	Guilin Yang, \emph{Member,~IEEE},\\
	Chenguang Yang, \emph{Fellow,~IEEE},
	and Zaojun Fang
\thanks{\textbf{This article has been accepted for publication in IEEE Transactions on Industrial Informatics. DOI: 10.1109/TII.2026.3738846.} (\textit {Corresponding authors: Silu Chen; Zaojun Fang}).
}
\thanks{Chan Xu, Xiyu Chen, and Dexin Jiang are with the Zhejiang Key Laboratory of Precision Actuation and Intelligent Robotics, Ningbo Institute of Materials Technology and Engineering, Chinese Academy of Sciences, Ningbo 315201, China, and also with the University of Chinese Academy of Sciences, Beijing 100049, China (e-mail: xuchan@nimte.ac.cn; chenxiyu@nimte.ac.cn; jiangdexin@nimte.ac.cn).
}
\thanks{Silu Chen, Dehao Wang, Chi Zhang, Guilin Yang, and Zaojun Fang are with the Zhejiang Key Laboratory of Precision Actuation and Intelligent Robotics, Ningbo Institute of Materials Technology and Engineering, Chinese Academy of Sciences, Ningbo 315201, China (e-mail: chensilu@nimte.ac.cn; wangdeh@nimte.ac.cn; zhangchi@nimte.ac.cn; glyang@nimte.ac.cn; fangzaojun@nimte.ac.cn).
}
\thanks{Chenguang Yang is with the Department of Computing, The Hong Kong Polytechnic University, HKSAR, China (e-mail: cyang@ieee.org).
}
\thanks{\copyright{} 2026 IEEE. Personal use of this material is permitted.
	Permission from IEEE must be obtained for all other uses, in any
	current or future media, including reprinting/republishing this
	material for advertising or promotional purposes, creating new
	collective works, for resale or redistribution to servers or lists,
	or reuse of any copyrighted component of this work in other works.}
}	

\maketitle
\begin{abstract}
	Dynamic Movement Primitives (DMPs) provide a compact and stable formulation for trajectory representation and generalization in robot skill learning. However, their predefined basis layout limits the allocation of approximation capacity according to stage-dependent precision requirements. To address this issue, this article proposes Stage-Criticality-Guided Dynamic Movement Primitives (SC-DMPs) with adaptive basis allocation for precision-critical skill learning. Operator-robot interaction stiffness and a trajectory-consistency cue derived from cross-demonstration task-space variability are integrated to construct a stage-criticality index. Guided by this index, basis centers are redistributed in normalized time through inverse cumulative criticality and mapped to the canonical phase domain, while their bandwidths are refined to adjust local approximation support. This enables denser and more flexible representation at high-criticality stages while retaining sparser allocation elsewhere. Experiments on handwriting trajectories and three real-robot tasks show that the inferred criticality is concentrated in geometrically demanding and task-constrained regions. Comparisons with DMPs, ProMPs, ProDMP, GP-MP, and KMP demonstrate improved trajectory reproduction, endpoint generalization, and task-critical accuracy while retaining a compact model and the stable structure of classical DMPs.
\end{abstract}
\begin{IEEEkeywords}
	Interaction stiffness, basis allocation, dynamic movement primitives, learning from demonstration.
\end{IEEEkeywords}
\maketitle

\section{Introduction}
\IEEEPARstart{L}{earning} from Demonstration (LfD) has emerged as an effective paradigm for transferring dexterous human motion skills to robots, especially in settings where explicit programming is difficult or costly~\cite{yu2025human}. Representative applications include fine insertion, tight-clearance assembly, and other operations with stringent precision requirements~\cite{hu2023robot}. In such tasks, the robot must not only reproduce the overall shape of the demonstrated trajectory, but also maintain high geometric accuracy at critical local stages, which are often less tolerant to error than other segments~\cite{bian2023learning}. Accordingly, achieving global fidelity and local accuracy at task-critical stages is central to precision-oriented trajectory learning~\cite{huang2025task}.

Dynamic Movement Primitives (DMPs) are one of the most representative trajectory-encoding frameworks in LfD, owing to their compact dynamical formulation, stability guarantees, and convenient spatiotemporal modulation~\cite{sun2022type}. Classical DMPs represent demonstrated motions through stable dynamical systems, where nonlinear forcing terms encode task-specific trajectory profiles~\cite{ijspeert2013dynamical}. In particular, the forcing term is approximated by basis functions preset along the canonical phase variable, where the basis layout is associated with the resulting reproduction performance~\cite{wang2025pi2}. Although this design is generally effective, it provides limited stage-wise adaptability for precision-critical skill representation~\cite{Zhang2023Robot}. Trajectory adaptability and tracking accuracy can be improved by integrating DMPs with model predictive control, reinforcement learning, and adaptive or neural feedback~\cite{li2025modified,ma2025model,yang2019robot}. However, such approaches may require additional environmental information or interaction, online computation, or specialized controller design, while local approximation errors introduced during trajectory encoding are not directly addressed~\cite{zhang2024using}. 

By contrast, representation-level methods modify movement primitives to improve trajectory reproduction and generalization, primarily through probabilistic or kernel-based formulations and parameter refinement. The former represents demonstrated trajectories and their uncertainty through probability distributions or kernel mappings, as in ProMPs~\cite{paraschos2013probabilistic}, ProDMP~\cite{li2023prodmp}, GP-MP~\cite{jin2023gaussian}, and KMP~\cite{huang2019kernelized}. However, local reproduction fidelity is typically regulated through covariance structures, kernel design, or hyperparameter settings, without explicitly allocating approximation support according to stage-dependent precision requirements. The latter refines movement-primitive parameterization by adjusting the number, locations, widths, or forms of basis functions, providing more explicit control over local approximation support~\cite{braglia2025phase,stulp2024fitting}. Among these refinements, basis-center redistribution relocates approximation support without increasing the basis number or altering the stable DMP dynamics, making it an efficient strategy for precision-critical skill learning~\cite{shi2025exploring}. Existing methods refine trajectory representations based on curvature, variance, recovered temporal structure, or mutual-information-based timestep weights~\cite{rozo2022orientation,jankowski2022key,colome2022mutual}. These criteria characterize trajectory structure or informativeness but do not necessarily indicate task-imposed motion tolerances. For example, peg-in-hole insertion may follow a geometrically simple path yet require high positioning accuracy under tight clearances~\cite{moreno2024obstacles}.

In precision-oriented physical interaction, increased stiffness may accompany stronger regulation of positional deviations, providing task information that is not evident from trajectory geometry alone~\cite{wen2023compliance,biagiotti2023robot}. Robot-environment contact stiffness can reveal such constraints during sustained contact, but its force-displacement-based estimation may be unreliable in noncontact, weak-contact, or contact-transition stages~\cite{zeng2024hierarchical,xing2023dynamic}. Operator-robot interaction stiffness can instead be estimated from muscle activation and upper-limb configuration without requiring robot-environment contact-force measurements, making it available across different stages of kinesthetic teaching~\cite{wang2025upper,xu2026extreme}. Existing stiffness-aware skill learning methods primarily incorporate stiffness or compliance into task execution through variable-impedance profiles, compliant dynamical systems, or programmed interaction behaviors~\cite{meattini2025programming,ge2024learning,wu2023impedance}. Alternatively, operator-robot interaction stiffness can serve as a task-informed cue for trajectory representation. Although informative, it may not fully characterize task criticality on its own, since some dynamic or periodic manipulation tasks may instead depend on compliant behavior~\cite{nah2023learning,liu2025lamps}. Accordingly, operator-robot interaction stiffness and cross-demonstration trajectory consistency may offer complementary cues for characterizing stage-dependent reproduction requirements.

In this work, operator-robot interaction stiffness and a trajectory-consistency cue derived from cross-demonstration variability are integrated to construct a stage-criticality index. Based on this index, Stage-Criticality-Guided Dynamic Movement Primitives (SC-DMPs) redistribute basis centers and adapt their bandwidths to allocate approximation support according to stage-dependent reproduction requirements. The main contributions are as follows.
\begin{enumerate}
	\item A physically motivated stage-criticality index is formulated by integrating operator-robot interaction stiffness and cross-demonstration task-space variability, providing an interpretable estimate of stage-wise reproduction-fidelity requirements.
	\item A criticality-guided basis allocation strategy is developed by redistributing centers through inverse cumulative criticality and adapting bandwidths to refine local support in high-criticality stages.
	\item Comprehensive experiments on handwriting trajectories and real-robot tasks compare SC-DMPs with DMPs, ProMPs, ProDMP, GP-MP, and KMP, demonstrating improved trajectory reproduction, generalization, and task-critical accuracy while retaining a compact model.
\end{enumerate}

The rest of this article is organized as follows. Section~\ref{PMD} formulates a probabilistic model of multiple demonstrations. Section~\ref{S_SC-DMPs} develops the stage-criticality index and SC-DMPs method. Section~\ref{Experiments} validates the effectiveness of the proposed method. Finally, Section~\ref{Conclusion} concludes this article.

\section{Probabilistic Modeling of Demonstrations}\label{PMD}
\begin{table}[t]
	\centering
	\caption{Summary of main symbols.}
	\label{Tab_symbols}
	\footnotesize
	\renewcommand{\arraystretch}{0.98}
	\setlength{\tabcolsep}{3.4pt}
	\begin{tabularx}{\columnwidth}{
			@{}
			>{\raggedright\arraybackslash}p{0.295\columnwidth}
			>{\raggedright\arraybackslash}X
			@{}
		}
		\toprule
		{Symbol} & {Description} \\
		\midrule
		$M$, $T$, $D$
		& Number of demonstrations, number of aligned time steps, and task-space dimension, respectively. \\
		$t_i$
		& Normalized time at the $i$-th aligned time step. \\
		$\bm{P}_{i,m}$, $\bm{K}_{i,m}$, $\bm{Q}_{i,m}$
		& Task-space position, axis-wise interaction stiffness, and their concatenated motion-state vector. \\
		$N_{\rm G}$, $\pi_j$, $\bm{\mu}_j$, $\bm{\Lambda}_j$
		& Number of Gaussian components and the weight, mean, and covariance of the $j$-th component. \\
		$\gamma_j(t)$, $\bm{\Lambda}_{Qt,j}$, $\bm{\Lambda}_{QQ,j}$
		& Posterior responsibility and covariance blocks used in GMR conditioning. \\
		$\bar{\bm{P}}(t)$, $\bar{\bm{K}}(t)$, $\bm{\sigma}_{P}^{2}(t)$
		& GMR reference trajectory, stiffness profile, and task-space variability. \\
		$\tilde{\bm{K}}_i$, $\tilde{\bm{v}}_i$, $\beta_1$, $\beta_2$
		& Normalized stiffness and variability descriptors and their fusion weights. \\
		$\bm{p}_i$, $\bm{r}_i$, $C_i$
		& Dimension-wise prior, temporally refined score, and scalar stage-criticality index. \\
		$s_i$, $N$
		& Canonical phase variable and number of DMP basis functions. \\
		$\Psi_n(s)$, $c_n$, $h_n$, $\bm{w}_n$
		& Gaussian basis function and its center, bandwidth, and weight vector. \\
		$F_i$
		& Cumulative stage criticality used for basis-center redistribution. \\
		$h_n^{(0)}$, $b_n$, $\mathcal{B}$
		& Reference bandwidth, bandwidth scaling factor, and candidate scaling set. \\
		$\bm{b}$, $\mathcal{J}(\bm{b})$
		& Bandwidth-scaling vector and forcing-term reconstruction objective. \\
		\bottomrule
	\end{tabularx}
	\vspace{-0.6em}
\end{table}

For a given task, $M$ demonstrations are collected, each consisting of a sequence of time stamps and corresponding motion observations. Dynamic Time Warping (DTW) is applied to temporally align these demonstrations to a common length $T$. The aligned time axis is then normalized to $[0,1]$, and the resulting dataset is represented as
\begin{equation}\label{Eq_D}
	\mathcal{D} =
	\left\{
	\left\{ (t_i,\bm{Q}_{i,m}) \right\}_{i=1}^{T}
	\right\}_{m=1}^{M},
	\,\,\,
	t_i=\frac{i-1}{T-1},
\end{equation}
where $t_i\in[0,1]$ is the normalized time at the $i$-th aligned time step, and $\bm{Q}_{i,m}\in \mathbb{R}^{2D}$ is the motion-state vector defined as
\begin{equation}\label{Eq_bmq}
	\bm{Q}_{i,m} =
	\left[
	\bm{P}_{i,m}^{\top}\;\;
	\bm{K}_{i,m}^{\top}
	\right]^{\top}.
\end{equation}
Here, $\bm{P}_{i,m}\in\mathbb{R}^{D}$ is the task-space position, $\bm{K}_{i,m}\in\mathbb{R}^{D}$ is the axis-wise interaction stiffness formed by the diagonal entries of the estimated endpoint stiffness matrix, and $D$ is the task-space dimension. The augmented vector is defined as
\begin{equation}\label{Eq_xi}
	\bm{\xi} =
	[t,\bm{Q}^{\top}]^{\top}
	\in \mathbb{R}^{1+2D},
\end{equation}
which serves as the basis for probabilistic encoding. For ease of reference, the main symbols used in the probabilistic modeling and SC-DMPs formulation are summarized in Table~\ref{Tab_symbols}.

\begin{figure*}
	\centering
	\includegraphics[width=0.98\textwidth]{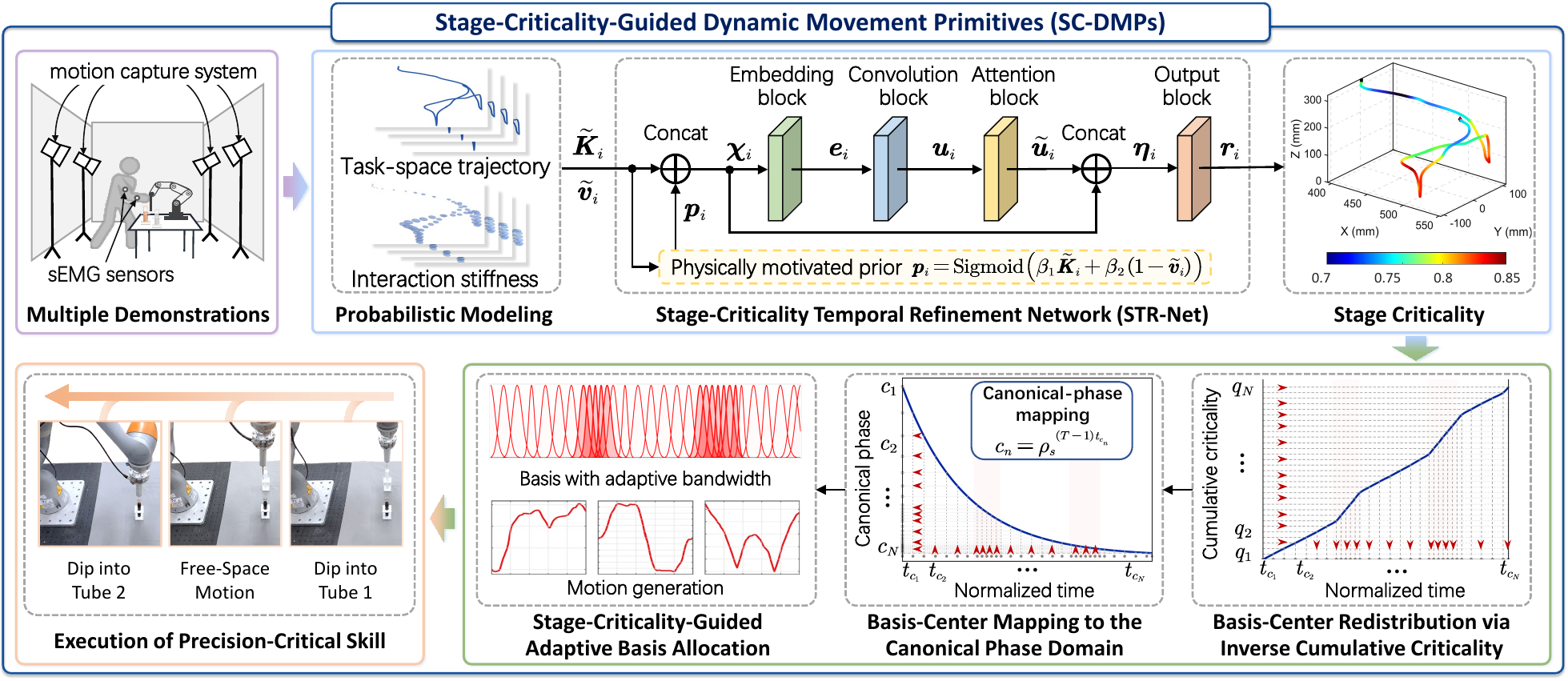}\\
	\caption{Overall framework of the proposed SC-DMPs. Operator-robot interaction stiffness and cross-demonstration task-space variability are used to estimate a stage-criticality index, which guides basis-center redistribution and adaptive bandwidth design for precision-critical skill reproduction.}
	\label{Fig_main}
	\vspace{-0.6em}
\end{figure*}

To capture the correlations among temporal evolution, task-space motion, and interaction stiffness, a Gaussian mixture model (GMM) is adopted to represent the joint distribution
\begin{equation}\label{Eq_GMM}
	p(\bm{\xi})=
	\sum_{j=1}^{N_{\rm G}}
	\pi_j\,
	\mathcal{N}\!\left(
	\bm{\xi};
	\bm{\mu}_j,\bm{\Lambda}_j
	\right),
\end{equation}
where $N_{\rm G}$ is the number of Gaussian components, and $\pi_j$ is the corresponding mixture weight satisfying $\sum_{j=1}^{N_{\rm G}}\pi_j=1$. The mean vector and covariance matrix of the $j$-th component are partitioned as
\begin{equation}\label{Eq_mu_partition}
	\bm{\mu}_j =
	\begin{bmatrix}
		\mu_{t,j}\\
		\bm{\mu}_{P,j}\\
		\bm{\mu}_{K,j}
	\end{bmatrix},
\end{equation}
and
\begin{equation}\label{Eq_lambda_partition}
	\bm{\Lambda}_j=
	\begin{bmatrix}
		\lambda_{tt,j} & \bm{\Lambda}_{tP,j} & \bm{\Lambda}_{tK,j}\\
		\bm{\Lambda}_{Pt,j} & \bm{\Lambda}_{PP,j} & \bm{\Lambda}_{PK,j}\\
		\bm{\Lambda}_{Kt,j} & \bm{\Lambda}_{KP,j} & \bm{\Lambda}_{KK,j}
	\end{bmatrix},
\end{equation}
where $\lambda_{tt,j}\in\mathbb{R}$ is the time variance of the $j$-th component.

Gaussian mixture regression (GMR) is then employed to infer the conditional distribution of $\bm{Q}$ given $t$:
\begin{equation}\label{Eq_GMR}
	p(\bm{Q}\mid t)=
	\sum_{j=1}^{N_{\rm G}}
	\gamma_j(t)\,
	\mathcal{N}\!\left(
	\bm{Q};
	\bm{\mu}_{Q\mid t,j}(t),
	\bm{\Lambda}_{Q\mid t,j}
	\right),
\end{equation}
where the posterior responsibility of the $j$-th component is
\begin{equation}\label{Eq_hg}
	\gamma_j(t)=
	\frac{
		\pi_j\,\mathcal{N}(t;\mu_{t,j},\lambda_{tt,j})
	}{
		\sum_{p=1}^{N_{\rm G}}
		\pi_p\,\mathcal{N}(t;\mu_{t,p},\lambda_{tt,p})
	}.
\end{equation}

For compactness, define
\begin{equation}\label{Eq_compact_blocks}
	\bm{\Lambda}_{Qt,j}=
	\begin{bmatrix}
		\bm{\Lambda}_{Pt,j}\\
		\bm{\Lambda}_{Kt,j}
	\end{bmatrix},
	\,\,\,
	\bm{\Lambda}_{QQ,j}=
	\begin{bmatrix}
		\bm{\Lambda}_{PP,j} & \bm{\Lambda}_{PK,j}\\
		\bm{\Lambda}_{KP,j} & \bm{\Lambda}_{KK,j}
	\end{bmatrix},
\end{equation}
and
\begin{equation}
	\bm{\mu}_{Q,j}=
	\begin{bmatrix}
		\bm{\mu}_{P,j}\\
		\bm{\mu}_{K,j}
	\end{bmatrix}.
\end{equation}
Then the conditional mean $\bm{\mu}_{Q\mid t,j}(t)$ and conditional covariance $\bm{\Lambda}_{Q\mid t,j}$ of the $j$-th component are given by
\begin{equation}\label{Eq_cond_mean}
	\bm{\mu}_{Q\mid t,j}(t)=
	\bm{\mu}_{Q,j}+
	\bm{\Lambda}_{Qt,j}\lambda_{tt,j}^{-1}(t-\mu_{t,j}),
\end{equation}
and
\begin{equation}\label{Eq_cond_cov}
	\bm{\Lambda}_{Q\mid t,j}=
	\bm{\Lambda}_{QQ,j}-
	\bm{\Lambda}_{Qt,j}\lambda_{tt,j}^{-1}\bm{\Lambda}_{tQ,j}.
\end{equation}

By further partitioning
\begin{equation}
	\bm{\mu}_{Q\mid t,j}(t)=
	\begin{bmatrix}
		\bm{\mu}_{P\mid t,j}(t)\\
		\bm{\mu}_{K\mid t,j}(t)
	\end{bmatrix},
\end{equation}
the reference task-space trajectory and stiffness profile are obtained from the mixture expectation as
\begin{equation}\label{Eq_refxK}
	\bar{\bm{P}}(t)\!=\!
	\sum_{j=1}^{N_{\rm G}}\!
	\gamma_j(t)\bm{\mu}_{P\mid t,j}(t),
	\,\,\,
	\bar{\bm{K}}(t)\!=\!
	\sum_{j=1}^{N_{\rm G}}\!
	\gamma_j(t)\bm{\mu}_{K\mid t,j}(t).
\end{equation}
To characterize the task-space trajectory variability across demonstrations, let $\bm{\Lambda}_{P\mid t,j}$ denote the position block of $\bm{\Lambda}_{Q\mid t,j}$, and let $\lambda_{P,d\mid t,j}$ denote its $d$-th diagonal entry, where $d=1,2,\ldots,D$. The variance of $d$-th dimension is computed as
\begin{equation}\label{Eq_sigma_xr}
	\sigma_{P,d}^{2}(t)=
	\sum_{j=1}^{N_{\rm G}}
	\gamma_j(t)
	\left(
	\lambda_{P,d\mid t,j}+\mu_{P,d\mid t,j}^{2}(t)
	\right)
	-\bar{P}_{d}^{2}(t),
\end{equation}
where $\mu_{P,d\mid t,j}(t)$ and $\bar{P}_d(t)$ denote the $d$-th entries of $\bm{\mu}_{P\mid t,j}(t)$ and $\bar{\bm{P}}(t)$, respectively. The corresponding variance vector is
\begin{equation}\label{Eq_sigma_x_vec}
	\bm{\sigma}_{P}^{2}(t)=
	\left[
	\sigma_{P,1}^{2}(t),\sigma_{P,2}^{2}(t),\ldots,\sigma_{P,D}^{2}(t)
	\right]^{\top}\!\!\in\mathbb{R}^{D}.
\end{equation}

In~\eqref{Eq_refxK}, larger interaction-stiffness values indicate stronger active regulation of motion deviations, whereas smaller variances in~\eqref{Eq_sigma_xr}-\eqref{Eq_sigma_x_vec} indicate stronger cross-demonstration consistency. These quantities are therefore adopted as complementary task-informed cues for estimating stage-dependent reproduction requirements and guiding adaptive basis allocation.

\section{Stage-Criticality-Guided Adaptive DMPs}\label{S_SC-DMPs}
This section constructs a stage-criticality index from interaction stiffness and task-space trajectory variability to guide adaptive basis allocation in the canonical phase domain. The overall framework of SC-DMPs is shown in Fig.~\ref{Fig_main}.

\subsection{Physically Motivated Stage Criticality}
At the $i$-th aligned time step, define the interaction stiffness and task-space variability descriptors as
\begin{equation}\label{Eq_vi}
	\bar{\bm{K}}_i = \bar{\bm{K}}(t_i) \in \mathbb{R}^{D},
	\quad
	\bm{v}_i = \bm{\sigma}_{P}^{2}(t_i) \in \mathbb{R}^{D}.
\end{equation}
To ensure numerical consistency, both descriptors are normalized to $[0,1]$, yielding $\tilde{\bm{K}}_i$ and $\tilde{\bm{v}}_i$. Based on the physical intuition that task-critical phases are typically characterized by high interaction stiffness and low task-space variability, a dimension-wise stage-criticality prior is defined as
\begin{equation}\label{Eq_prior}
	\bm{p}_i =
	\operatorname{Sigmoid}
	\!\left(
	\beta_1 \tilde{\bm{K}}_i + \beta_2(\bm{1}-\tilde{\bm{v}}_i)
	\right)
	\in (0,1)^{D},
\end{equation}
where $\bm{1}\in\mathbb{R}^{D}$ is a vector of ones, and $\beta_1,\beta_2\ge 0$ weight the contributions of interaction stiffness and inverse task-space variability. The prior $\bm{p}_i$ provides a physically interpretable initial estimate of dimension-wise stage criticality.

Since~\eqref{Eq_prior} is evaluated independently at each aligned time step, local stiffness-estimation errors, misalignment, and measurement noise may lead to criticality fluctuations. To mitigate these isolated variations, a lightweight Stage-Criticality Temporal Refinement Network (STR-Net) is introduced to refine the physically motivated prior by incorporating temporal context. First, the input feature at step $i$ is defined as
\begin{equation}\label{Eq_input_feature}
	\bm{\chi}_i =
	\left[
	\tilde{\bm{K}}_i^{\top}\;\;
	\tilde{\bm{v}}_i^{\top}\;\;
	\bm{p}_i^{\top}
	\right]^{\top}\!
	\in \mathbb{R}^{3D}.
\end{equation}
The input is projected into a latent space through
\begin{equation}\label{Eq_embedding}
	\bm{e}_i =
	\varphi
	\!\left(
	\bm{W}_{\rm emb}\bm{\chi}_i+\bm{\theta}_{\rm emb}
	\right)
	\in \mathbb{R}^{d_{\rm model}},
\end{equation}
where $\bm{W}_{\rm emb} \!\in\! \mathbb{R}^{d_{\rm model}\times 3D}$ and $\bm{\theta}_{\rm emb} \!\in\! \mathbb{R}^{d_{\rm model}}$ are learnable parameters, and $\varphi(\cdot)$ denotes an element-wise nonlinear activation function. To capture short-range temporal dependencies, the embedding sequence $\{\bm{e}_i\}_{i=1}^{T}$ is processed by a convolutional layer
\begin{equation}\label{Eq_conv}
	\bm{u}_i =
	\operatorname{Conv}
	\big(
	[\bm{e}_1,\ldots,\bm{e}_T]
	\big)_i
	\in \mathbb{R}^{d_{\rm model}}.
\end{equation}
A lightweight attention mechanism is then used to emphasize globally critical time steps
\begin{equation}\label{Eq_attention}
	\alpha_i^{\rm att} =
	\frac{
		\exp\!\left(
		\bm{a}_{\rm att}^{\top}
		\tanh(\bm{W}_{\rm att}\bm{u}_i)
		\right)
	}{
		\sum_{p=1}^{T}
		\exp\!\left(
		\bm{a}_{\rm att}^{\top}
		\tanh(\bm{W}_{\rm att}\bm{u}_p)
		\right)
	},
\end{equation}
where $\bm{W}_{\rm att}$ and $\bm{a}_{\rm att}$ are learnable attention parameters. The context-enhanced latent feature is then written as
\begin{equation}\label{Eq_context}
	\tilde{\bm{u}}_i =
	(1+\alpha_i^{\rm att})\bm{u}_i.
\end{equation}

To keep the refinement anchored to the original physical cues, the output layer uses the fused representation
\begin{equation}\label{Eq_fused}
	\bm{\eta}_i =
	\left[
	\tilde{\bm{u}}_i^{\top}\;\;
	\tilde{\bm{K}}_i^{\top}\;\;
	\tilde{\bm{v}}_i^{\top}\;\;
	\bm{p}_i^{\top}
	\right]^{\top}
	\in \mathbb{R}^{d_{\rm model}+3D},
\end{equation}
and predicts refined stage-criticality scores
\begin{equation}\label{Eq_ri}
	\bm{r}_i =
	\operatorname{Sigmoid}
	\!\left(
	\bm{W}_{\rm out}^{\top}\bm{\eta}_i+\bm{\theta}_{\rm out}
	\right)
	\in (0,1)^{D},
\end{equation}
where $\bm{W}_{\rm out}\in\mathbb{R}^{(d_{\rm model}+3D)\times D}$ and $\bm{\theta}_{\rm out}\in\mathbb{R}^{D}$.

The STR-Net is trained to refine the initial physically motivated estimate under consistency and temporal-smoothness
\begin{equation}\label{Eq_loss}
	\mathcal{L} \!=\!
	\frac{1}{T}
	\sum_{i=1}^{T}\sum_{d=1}^{D}
	\ell_{\rm Huber}(r_{i,d}\!-\!p_{i,d})
	\!+\!
	\frac{\zeta}{T\!-\!1}
	\sum_{i=1}^{T\!-\!1}
	\left\|
	\bm{r}_{i+1}\!-\!\bm{r}_i
	\right\|_2^2,
\end{equation}
where $\ell_{\rm Huber}(\cdot)$ is the Huber loss, and $\zeta>0$ controls the strength of smoothness. Notably, STR-Net does not replace the physically motivated prior, but regularizes it in the temporal domain to suppress isolated fluctuations while preserving coherent criticality intervals. 

Then, the dimension-wise scores refined by STR-Net are aggregated into a scalar stage-criticality index.
\begin{definition}[Stage-Criticality Index]
	For the refined score vector $\bm{r}_i \in (0,1)^{D}$ at the $i$-th aligned time step, the stage-criticality index is defined as
	\begin{equation}\label{Eq_Ci_prop}
		C_i = \frac{1}{D}\sum_{d=1}^{D} r_{i,d}.
	\end{equation}
\end{definition}

Since $\bm{r}_i\in(0,1)^D$, the averaged index satisfies $C_i\in(0,1)$ and provides a discrete stage-criticality estimate at the $i$-th aligned time step. A larger $C_i$ suggests a stronger stage-wise reproduction-fidelity requirement.

\subsection{Stage-Criticality-Guided Basis Allocation}
For a $D$-dimensional task-space trajectory, the classical discrete-time DMP formulation is given by
\begin{subequations}\label{Eq_DMP}
	\begin{align}
		\bm{z}_{i+1} &= \bm{z}_{i} + \frac{\Delta t}{\tau} \left[ \alpha_z \left( \beta_z(\bm{g}-\bm{P}_{i}) - \bm{z}_{i} \right) + \bm{f}(s_i) \right],
		\label{eq:residual_dmp_discrete_a} \\
		\bm{P}_{i+1} &= \bm{P}_{i} + \frac{\Delta t}{\tau} \bm{z}_{i+1},
		\label{eq:residual_dmp_discrete_b} \\
		s_{i+1} &= s_i - \frac{\Delta t}{\tau} \alpha_s s_i,
		\label{eq:residual_dmp_discrete_c}
	\end{align}
\end{subequations}
where $\bm{P},\bm{z},\bm{g}\in\mathbb{R}^{D}$ denote the task-space position, velocity-related auxiliary state, and goal position, respectively; $s$ is the canonical phase variable; $\tau$ is the temporal scaling factor; $\Delta t$ is the physical sampling interval; and $\alpha_z,\beta_z,\alpha_s>0$ are positive constants. The nonlinear forcing term is defined as a function of the canonical phase variable
\begin{equation}\label{Eq_fx}
	\bm{f}(s) =
	\left(
	\frac{\sum_{n=1}^{N} \bm{w}_n \Psi_n(s)}
	{\sum_{n=1}^{N} \Psi_n(s)}
	\right)s,
\end{equation}
with Gaussian basis functions
\begin{equation}\label{Eq_psi}
	\Psi_n(s)=
	\exp\!\left(-h_n(s-c_n)^2\right),
	\quad n=1,\ldots,N,
\end{equation}
where $c_n$, $h_n$, and $\bm{w}_n\in\mathbb{R}^{D}$ denote the center, bandwidth, and weight vector of the $n$-th basis function, respectively.

Instead of uniformly placing basis centers in normalized time, SC-DMPs redistribute them according to accumulated stage criticality. By assigning equal accumulated criticality to adjacent center intervals, regions with larger $C_i$ receive denser basis support. The following cumulative profile is introduced.
\begin{definition}[Cumulative Criticality]\label{Def_cumulative_criticality}
	The cumulative criticality on the aligned normalized-time grid is defined as
	\begin{equation}\label{Eq_Fs}
		F_1 = 0,\quad
		F_{i+1} = F_i + C_i,\quad i = 1,2,\ldots,T-1 .
	\end{equation}
\end{definition}
Since $C_i\in(0,1)$, the sequence $\{F_i\}_{i=1}^{T}$ is strictly increasing. Moreover, recursively applying~\eqref{Eq_Fs} yields
	\begin{equation}
		F_T
		= F_1+\sum_{i=1}^{T-1}C_i
		= \sum_{i=1}^{T-1}C_i.
	\end{equation}
Thus, $F_i$ provides a monotonic cumulative profile of stage criticality for subsequent basis-center redistribution.

To determine the basis centers, the cumulative-criticality range $[0,F_T]$ is uniformly partitioned into $N-1$ subintervals. The target cumulative-criticality level associated with the $n$-th basis center is defined as
\begin{equation}\label{Eq_qn}
	q_n =
	\frac{n-1}{N-1}F_T,
	\quad n=1,2,\ldots,N .
\end{equation}
For each $q_n$, let $k_n\in\{1,2,\ldots,T-1\}$ be the index satisfying
\begin{equation}\label{Eq_kn}
	F_{k_n}\le q_n \le F_{k_n+1}.
\end{equation}
The relative position of $q_n$ within this interval is measured by
\begin{equation}\label{Eq_eta_n}
	\eta_n =
	\frac{q_n-F_{k_n}}{F_{k_n+1}-F_{k_n}},
	\quad n=1,2,\ldots,N .
\end{equation}
The corresponding basis center in normalized time is then obtained by linear interpolation
\begin{equation}\label{Eq_tcenter}
	t_{c_n} = (1-\eta_n)t_{k_n} + \eta_n t_{k_n+1},
	\quad n=1,2,\ldots,N .
\end{equation}
Since the target cumulative-criticality levels $\{q_n\}_{n=1}^{N}$ are uniformly spaced, adjacent basis-center intervals carry equal accumulated criticality. Consequently, regions with larger $C_i$ receive denser basis support in normalized time.

The basis centers are then mapped to the canonical phase domain according to the discrete canonical dynamics in~\eqref{eq:residual_dmp_discrete_c}. Let
\begin{equation}\label{Eq_lambda}
	\rho_s =
	1-\frac{\Delta t}{\tau}\alpha_s,
	\quad 0<\rho_s <1 .
\end{equation}
Starting from $s_1=1$, the phase sequence induced by~\eqref{eq:residual_dmp_discrete_c} is
\begin{equation}\label{Eq_si_discrete}
	s_i = \rho_s^{i-1}
	= \rho_s^{(T-1)t_i},
	\quad i=1,2,\ldots,T .
\end{equation}
Each criticality-guided center $t_{c_n}$ in normalized time is mapped to the corresponding canonical phase center as
\begin{equation}\label{Eq_centers}
	c_n =
	\rho_s^{(T-1)t_{c_n}},
	\quad n=1,2,\ldots,N .
\end{equation}
Since $0<\rho_s<1$, this mapping is strictly decreasing with respect to normalized time. Hence, the resulting centers satisfy
\begin{equation}
	1=c_1>c_2>\cdots>c_N=s_T,
\end{equation}
where
\begin{equation}
	s_T=\rho_s^{T-1}.
\end{equation}
This mapping preserves the criticality-guided allocation obtained in normalized time while making the basis centers consistent with the canonical phase evolution of DMPs.

The proposed allocation reduces to the standard center placement of classical DMPs under constant stage criticality.
\begin{corollary}[Reduction to Standard Center Placement]\label{Cor_classic}
	If the stage-criticality sequence is constant, i.e.,
	$C_i\equiv C_0>0$ for $i=1,2,\ldots,T-1$, then the criticality-guided centers in normalized time reduce to the uniformly spaced centers
	\begin{equation}\label{Eq_classic_time_center}
		t_{c_n}=\frac{n-1}{N-1},
		\quad n=1,2,\ldots,N .
	\end{equation}
	The corresponding centers in the canonical phase domain are
	\begin{equation}\label{Eq_classic_center}
		c_n=
		\rho_s^{(T-1)\frac{n-1}{N-1}},
		\quad n=1,2,\ldots,N .
	\end{equation}
\end{corollary}
\begin{proof}
	When $C_i\equiv C_0$, the cumulative criticality becomes
	\begin{equation}
		F_i=(i-1)C_0=F_T t_i,
		\quad F_T=(T-1)C_0 .
	\end{equation}
	Thus, the cumulative criticality is linear with respect to the normalized time. For the target cumulative-criticality level $q_n$ in~\eqref{Eq_qn}, the inverse cumulative mapping gives
	\begin{equation}\label{Eq_classic_1122}
		t_{c_n}=\frac{q_n}{F_T}
		=
		\frac{n-1}{N-1},
		\quad n=1,2,\ldots,N .
	\end{equation}
	This proves~\eqref{Eq_classic_time_center}. Substituting~\eqref{Eq_classic_1122} into the phase mapping in~\eqref{Eq_centers} yields~\eqref{Eq_classic_center}. This completes the proof.
\end{proof}

\subsection{Adaptive Bandwidth Design}
After the stage-criticality-guided centers are determined, the bandwidths are adapted to refine the local support and overlap of the basis functions. Here, the adjacent-center spacing provides a reference bandwidth, while additional scaling factors are introduced to improve local approximation flexibility.

For centers $\{c_n\}_{n=1}^{N}$ ordered in descending canonical phase, the reference bandwidth is defined as
\begin{equation}\label{Eq_h0}
	h_n^{(0)} =
	\frac{2}{(c_{n+1}-c_n)^2},
	\quad n=1,2,\ldots,N-1.
\end{equation}
For $n=N$, the reference bandwidth is set equal to that of the preceding basis function
\begin{equation}\label{Eq_hN}
	h_N^{(0)} = h_{N-1}^{(0)}.
\end{equation}
The effective bandwidth is then parameterized as
\begin{equation}\label{Eq_hn}
	h_n = b_n h_n^{(0)},
\end{equation}
where each scaling factor $b_n$ is selected from a predefined candidate set $\mathcal{B}$. Let
\begin{equation}\label{Eq_bvec}
	\bm{b} =
	[b_1,b_2,\ldots,b_N]^{\top}
\end{equation}
collect all bandwidth scaling factors.

To evaluate a bandwidth configuration $\bm{b}$, the target forcing term is first extracted from the reference trajectory. For the $d$-th task-space dimension, rearranging~\eqref{eq:residual_dmp_discrete_a} gives
\begin{equation}\label{Eq_target_force}
	\hat{f}_{d}(s_i)
	\!=\!
	\frac{\tau}{\Delta t}
	\left(
	z_{d,i+1}-z_{d,i}
	\right)
	\!-\!
	\alpha_z
	\left[
	\beta_z
	\left(
	g_d\!-\!P_{d,i}
	\right)
	\!-\!
	z_{d,i}
	\right],
\end{equation}
where $P_{d,i}$ and $z_{d,i}$ are obtained from the reference trajectory and its associated velocity-related state. For a given $\bm{b}$, the bandwidths are determined by~\eqref{Eq_hn}, and the corresponding basis functions are denoted by $\Psi_n(s_i;\bm{b})$. The reconstructed forcing term is written as
\begin{equation}\label{Eq_tildef}
	\tilde{f}_d(s_i;\bm{b})
	=
	\left(
	\frac{
		\sum_{n=1}^{N}
		w_{n,d}(\bm{b})\Psi_n(s_i;\bm{b})
	}{
		\sum_{n=1}^{N}
		\Psi_n(s_i;\bm{b})
	}
	\right)s_i ,
\end{equation}
where $w_{n,d}(\bm{b})$ denotes the fitted weight associated with the $n$-th basis function under the bandwidth configuration $\bm{b}$. The reconstruction cost is then defined as
\begin{equation}\label{Eq_J}
	\mathcal{J}(\bm{b})
	=
	\frac{1}{D(T-1)}
	\sum_{d=1}^{D}
	\sum_{i=1}^{T-1}
	\left|
	\hat{f}_d(s_i) - \tilde{f}_d(s_i;\bm{b})
	\right|^2 .
\end{equation}

Because overlapping Gaussian basis functions couple $\mathcal{J}(\bm{b})$ across the scaling factors, exhaustive evaluation of all joint candidates is costly. Therefore, $\bm{b}$ is optimized by cyclic coordinate descent. Let $\bar{\bm{b}}=[\bar{b}_1,\ldots,\bar{b}_N]^{\top}$ denote the current scaling vector; when updating its $n$-th entry, all other entries are fixed and each candidate $b\in\mathcal{B}$ is evaluated by defining
\begin{equation}\label{Eq_temp_b}
	\bm{b}^{(n)}(b)
	=
	[ \bar{b}_1,\ldots,\bar{b}_{n-1}, b, \bar{b}_{n+1},\ldots,\bar{b}_N
	]^{\top}.
\end{equation}
The coordinate-wise optimal candidate is selected by
\begin{equation}\label{Eq_coord}
	b_n^{\star}
	=
	\arg \, \min_{b\in\mathcal{B}}\,
	\mathcal{J}\!\left(\bm{b}^{(n)}(b)\right).
\end{equation}
The candidate $b_n^\star$ is accepted only if it reduces the current objective by more than $\delta$. Otherwise, $\bar b_n$ remains unchanged. The coordinate-wise updates are performed cyclically for $n=1,2,\ldots,N$, with each complete cycle constituting one sweep. The procedure terminates when a full sweep accepts no update or when the maximum number of sweeps $\Gamma$ is reached.

For bandwidth optimization, the finite candidate set yields $|\mathcal{B}|^{N}$ feasible scaling vectors. Each accepted update decreases $\mathcal{J}$ by more than $\delta$. This strict decrease prevents the algorithm from revisiting a previous scaling vector and bounds the total number of updates. If a sweep with no accepted update triggers termination before reaching the maximum sweep count $\Gamma$, the result is a coordinate-wise local minimum within tolerance $\delta$, although global optimality is not guaranteed.

\section{Experiments}\label{Experiments}
\subsection{Experimental Setup}
The system for identifying operator-robot interaction stiffness is shown in Fig.~\ref{Fig_ISIS}. Before data collection, four sEMG sensors are placed over the muscle bellies of the biceps brachii, triceps brachii, brachioradialis, and flexor carpi ulnaris, parallel to the muscle fibers and away from tendons and joints. During demonstrations, the operator guides a KUKA iiwa robot through predefined tasks. Robot joint positions and shoulder, elbow, and wrist positions are recorded nominally at approximately 100~Hz, while the sEMG signals are acquired at 2000~Hz. The end-effector trajectory is reconstructed through forward kinematics, the motion-capture data are smoothed using a 50-sample moving average, and each sEMG signal is demeaned, band-pass filtered at 20-350~Hz, full-wave rectified, and low-pass filtered at 5~Hz. The processed signals are then timestamp-aligned with the robot timeline and used to estimate the interaction stiffness according to~\cite{wang2025upper}. 

For STR-Net, we set $\beta_1=\beta_2=1$, $d_{\rm model}=64$, and the convolutional kernel size to 5. The model is trained for 300 epochs using AdamW with a learning rate of $1\times10^{-3}$. The loss combines a Huber-based prior-consistency term with a temporal smoothness term weighted by $\zeta=1$.

\begin{figure}
	\centering
	\includegraphics[width=0.75\columnwidth]{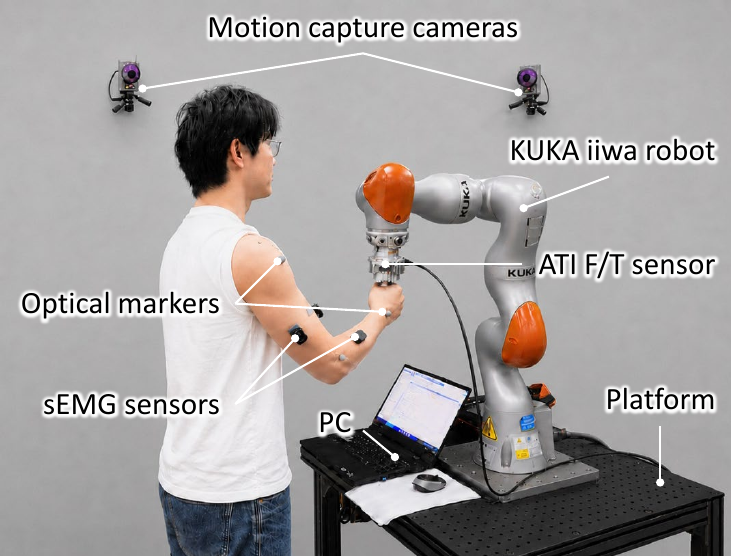}\\
	\caption{Experimental platform for interaction stiffness identification.}
	\label{Fig_ISIS}
	\vspace{-0.4em}
\end{figure}

Moreover, five evaluation metrics are employed to assess trajectory learning and generalization performance.
\begin{itemize}
	\item Root Mean Square Error (RMSE): RMSE measures the overall average pointwise deviation
		\begin{equation}
			\text{RMSE} = \sqrt{ \frac{1}{T} \sum_{i=1}^{T} \left\| \bm{P}(t_i)-\bm{P}^{\text{ref}}(t_i) \right\|_2^2 }.
		\end{equation}
	\item Weighted RMSE (W-RMSE): W-RMSE uses the inferred stage-criticality profile $C(t)$ as weights and is therefore used as an auxiliary criticality-weighted metric
		\begin{equation}
			\text{W-RMSE}  =\! \sqrt{ \sum_{i=1}^{T} \frac{C_i}{\sum_{k=1}^{T} C_k} \left\| \bm{P}(t_i)\!-\!\bm{P}^{\text{ref}}(t_i) \right\|_2^2 }.
		\end{equation}
	\item RMSE in Geometry-critical regions (RMSE-Geo): RMSE-Geo independently evaluates local errors over geometry-critical regions defined by the top 20\% curvature indices of the reference trajectory rather than $C(t)$:
	\begin{equation}
		\text{RMSE-Geo} \!=\!\!
		\sqrt{
			\frac{1}{|\mathcal{I}_{\text{geo}}|}
			\sum_{i\in \mathcal{I}_{\text{geo}}}\!
			\left\|
			\bm{P}(t_i)\!-\!\bm{P}^{\text{ref}}(t_i)
			\right\|_2^2
		}.
	\end{equation}
	\item Maximum Error (MaxE): MaxE captures the worst-case deviation by
	\begin{equation}
		\text{MaxE} = \max_{i=1,2,\ldots,T} \|\bm{P}(t_i)-\bm{P}^{\text {ref}}(t_i)\|_2.
	\end{equation}
	\item Normalized Shape Error (NSE): For generalization evaluation under varying start or target configurations, both the reproduced trajectory $\bm{P}(t_i)$ and the target trajectory $\bm{P}^{\text{tar}}(t_i)$ are normalized by
	\begin{equation}
		\tilde{\bm{P}}(t_i) \!=\!\! \frac{{\bm{P}}(t_i)\!-\!\bm{P}_{\text{start}}}{\|\bm{P}_{\text{end}}\!-\!\bm{P}_{\text{start}}\|}, \,\,
		\tilde{\bm{P}}^{\text{tar}}(t_i) \!=\!\! \frac{\bm{P}^{\text{tar}}(t_i)\!-\!\bm{P}^{\text{tar}}_{\text{start}}}{\|\bm{P}^{\text{tar}}_{\text{end}}\!-\!\bm{P}^{\text{tar}}_{\text{start}}\|}.
	\end{equation}
	Then, NSE is defined as
	\begin{equation}
		\text{NSE} = \frac{1}{T} \sum_{i=1}^{T} \left\| \tilde{\bm{P}}(t_i) - \tilde{\bm{P}}^{\text{tar}}(t_i) \right\|_2.
	\end{equation}
\end{itemize}

\begin{figure}
	\centering
	\includegraphics[width=0.955\columnwidth]{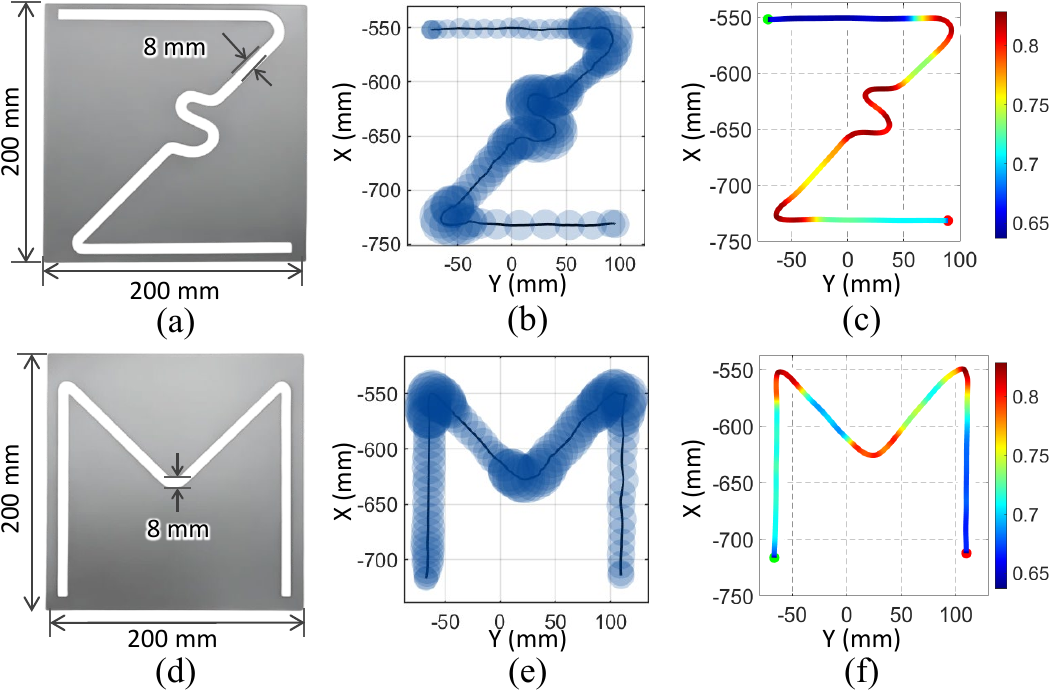}\\
	\caption{Handwriting letter data acquisition. (a)-(c) present the template, the trajectory with stiffness ellipsoids, and the visualized stage-criticality index for letter ``Z''. (d)-(f) show the corresponding results for letter ``M''.}
	\label{Fig_HLE_setup_2}
	\vspace{-0.5em}
\end{figure}

\begin{figure}
	\centering
	\includegraphics[width=0.93\columnwidth]{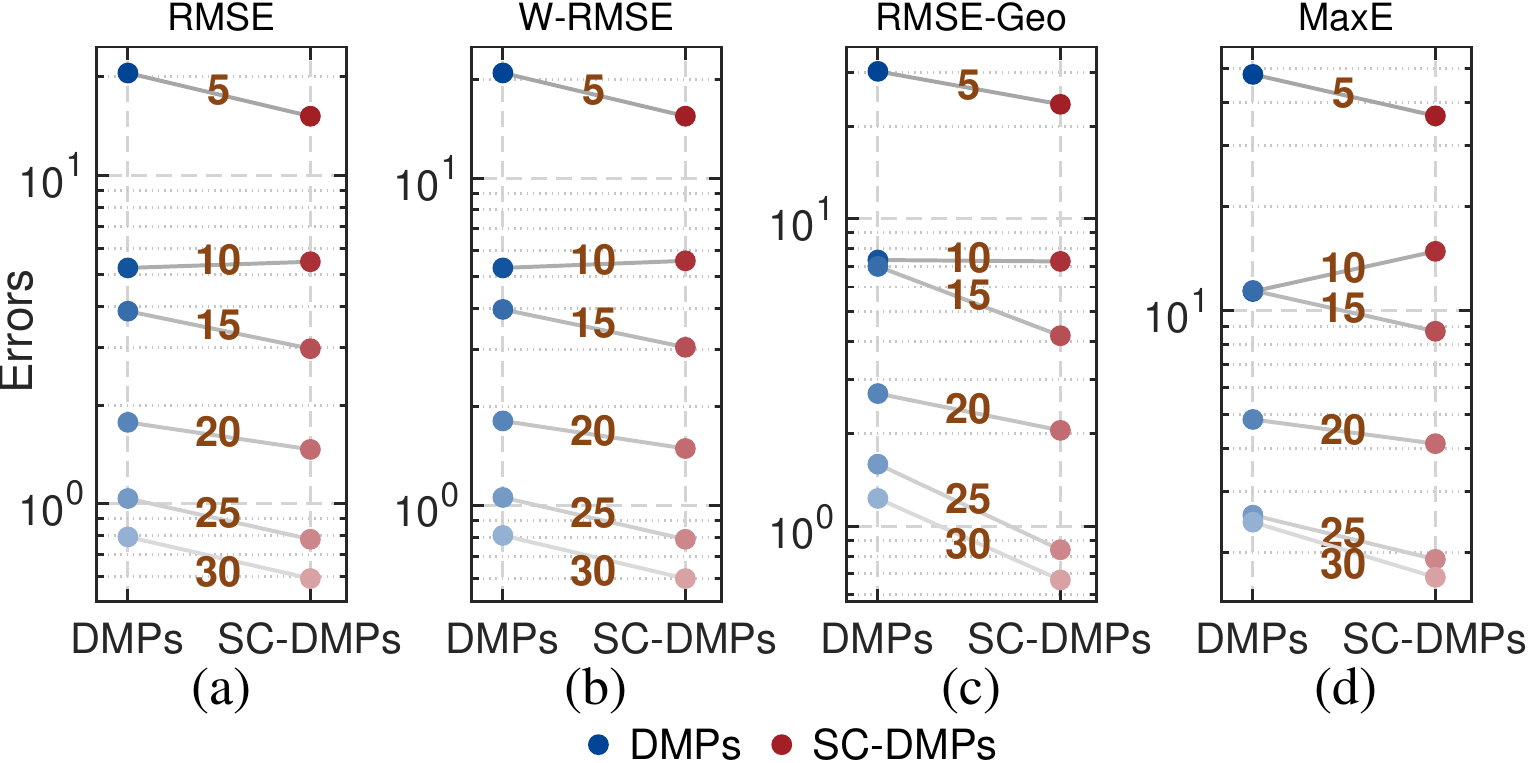}\\
	\caption{Learning errors of SC-DMPs and DMPs for ``Z'' under 5, 10, 15, 20, 25, and 30 basis functions on a logarithmic scale. (a)~RMSE (mm). (b) W-RMSE (mm). (c) RMSE-Geo (mm). (d) MaxE (mm).}
	\label{Fig_letter_Errors_Z}
	\vspace{-0.6em}
\end{figure}
\begin{figure}
	\centering
	\includegraphics[width=0.93\columnwidth]{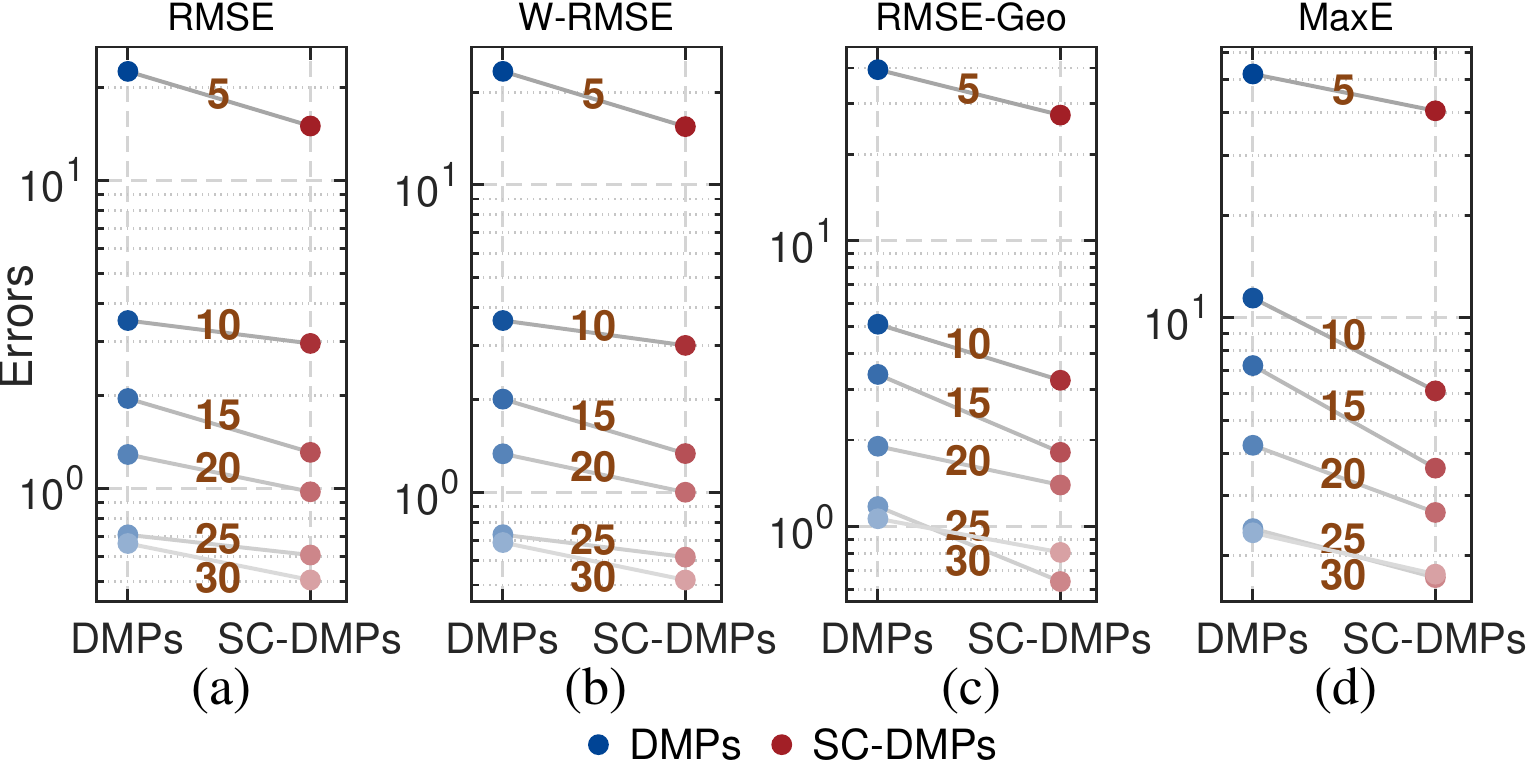}\\
	\caption{Learning errors of SC-DMPs and DMPs for ``M'' under 5, 10, 15, 20, 25, and 30 basis functions on a logarithmic scale. (a)~RMSE (mm). (b) W-RMSE (mm). (c) RMSE-Geo (mm). (d) MaxE (mm).}
	\label{Fig_letter_Errors_M}
	\vspace{-0.6em}
\end{figure}

\subsection{Handwriting Letter Experiment}
Two hollow templates shaped as the letters ``Z'' and ``M'' are used, as shown in Fig.~\ref{Fig_HLE_setup_2}(a) and (d), respectively. During each demonstration, a marker attached to the robot end-effector is manually guided along the letter path while remaining inside the template without touching its boundaries. Twelve demonstrations are collected for each letter, with the marker trajectory recorded and the corresponding interaction stiffness estimated, as shown in Fig.~\ref{Fig_HLE_setup_2}(b) and (e). The GMR reference trajectories obtained from~\eqref{Eq_refxK} and the corresponding stage-criticality profiles are presented in Fig.~\ref{Fig_HLE_setup_2}(c) and (f). For both letters, criticality peaks around turning regions, where small deviations can noticeably distort the letter shape.

\begin{table}[t]
	\centering
	\caption{Statistical analysis of SC-DMPs versus DMPs across different basis-function settings.}
	\renewcommand\arraystretch{0.95}
	\resizebox{0.91\columnwidth}{!}{
		\begin{tabular}{p{1.6cm}<{\centering}p{2.25cm}<{\centering}p{3.1cm}<{\centering}p{1.17cm}<{\centering}}
			\toprule
			Errors & Mean reduction & 95\%~confidence~interval & Adjusted~$p$ \\
			\midrule
			RMSE     & 21.42\% ± 10.11\% & [13.99\%, 25.65\%] & 0.0103  \\
			W-RMSE   & 21.92\% ± 10.27\% & [13.99\%, 26.08\%] & 0.0103  \\
			RMSE-Geo & 32.50\% ± 13.89\% & [23.55\%, 38.91\%] & 0.0020  \\
			MaxE     & 24.67\% ± 20.16\% & [8.57\%, 32.56\%] & 0.0122  \\
			\bottomrule
		\end{tabular}%
	}
	\vspace{-0.4em}
	\label{Tab_letter_Sta}
\end{table}

\subsubsection{Learning Performance and Basis-Allocation Analysis}
SC-DMPs are compared with DMPs using 5, 10, 15, 20, 25, and 30 basis functions. RMSE, W-RMSE, RMSE-Geo, and MaxE are used as evaluation metrics. The results are reported in Figs.~\ref{Fig_letter_Errors_Z} and~\ref{Fig_letter_Errors_M}, with logarithmically scaled vertical axes and brown annotations indicating the numbers of basis functions.

SC-DMPs achieve lower errors in 45 of the 48 letter-basis-metric comparisons. The lower RMSE indicates improved overall reproduction accuracy, while reductions in W-RMSE and RMSE-Geo reflect better fitting in stage-critical and geometrically demanding regions, respectively. The lower MaxE further demonstrates reduced worst-case local deviations. These improvements are most pronounced with fewer basis functions, indicating more efficient use of representational capacity. As the number increases, the performance gap narrows because the additional capacity partially compensates for the uniform center placement of conventional DMPs. Table~\ref{Tab_letter_Sta} summarizes the paired statistical analysis across the 12 letter-basis settings for each metric. SC-DMPs reduce the mean errors by 21.42\%-32.50\% relative to DMPs. All 95\% confidence intervals exclude zero, and all Wilcoxon signed-rank tests remain significant after Holm adjustment.

To illustrate basis allocation and learned trajectories under limited and relatively sufficient representational capacities, 15 and 30 basis functions are selected for visualization. As shown in Figs.~\ref{Fig_centers} and~\ref{Fig_xuexi}, SC-DMPs allocate denser centers to high-criticality regions and sparser ones elsewhere, thereby increasing local approximation support around bends and turning regions. Consequently, the reproduced trajectories better match the reference in these demanding segments.

\begin{figure}
	\centering
	\includegraphics[width=0.86\columnwidth]{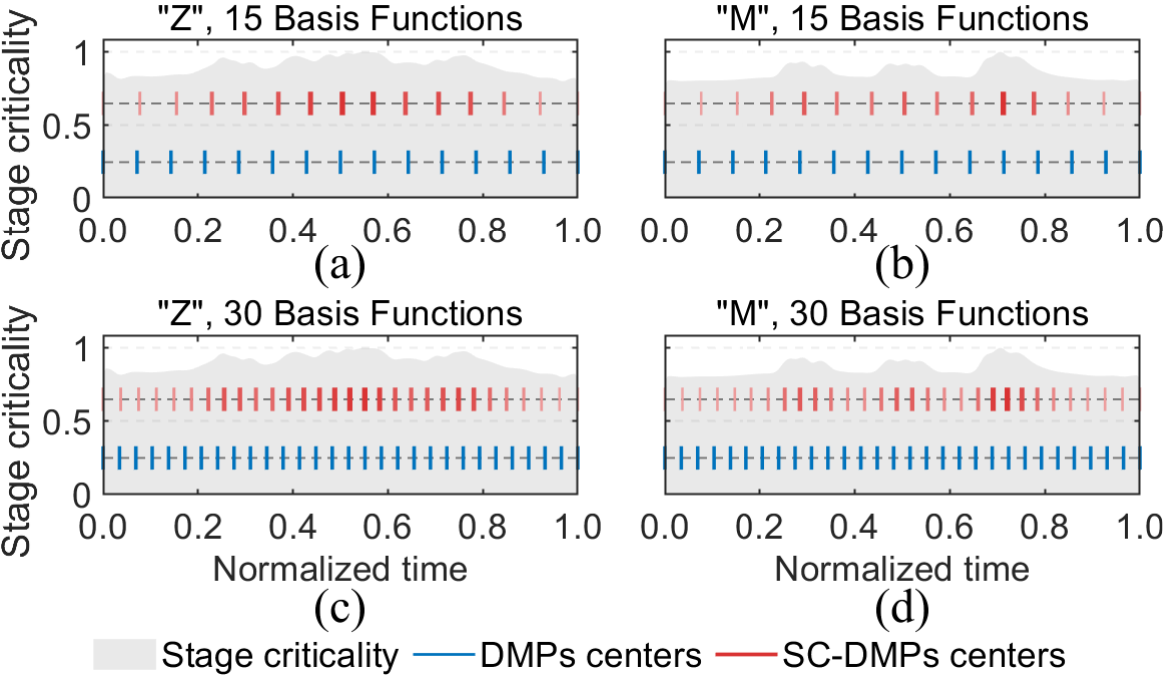}\\
	\caption{Center distributions under 15 and 30 basis functions, where color intensity increases with density. (a) and (c) correspond to ``Z''. (b)~and (d) correspond to ``M''.}
	\label{Fig_centers}
	\vspace{-0.6em}
\end{figure}
\begin{figure}
	\centering
	\includegraphics[width=0.93\columnwidth]{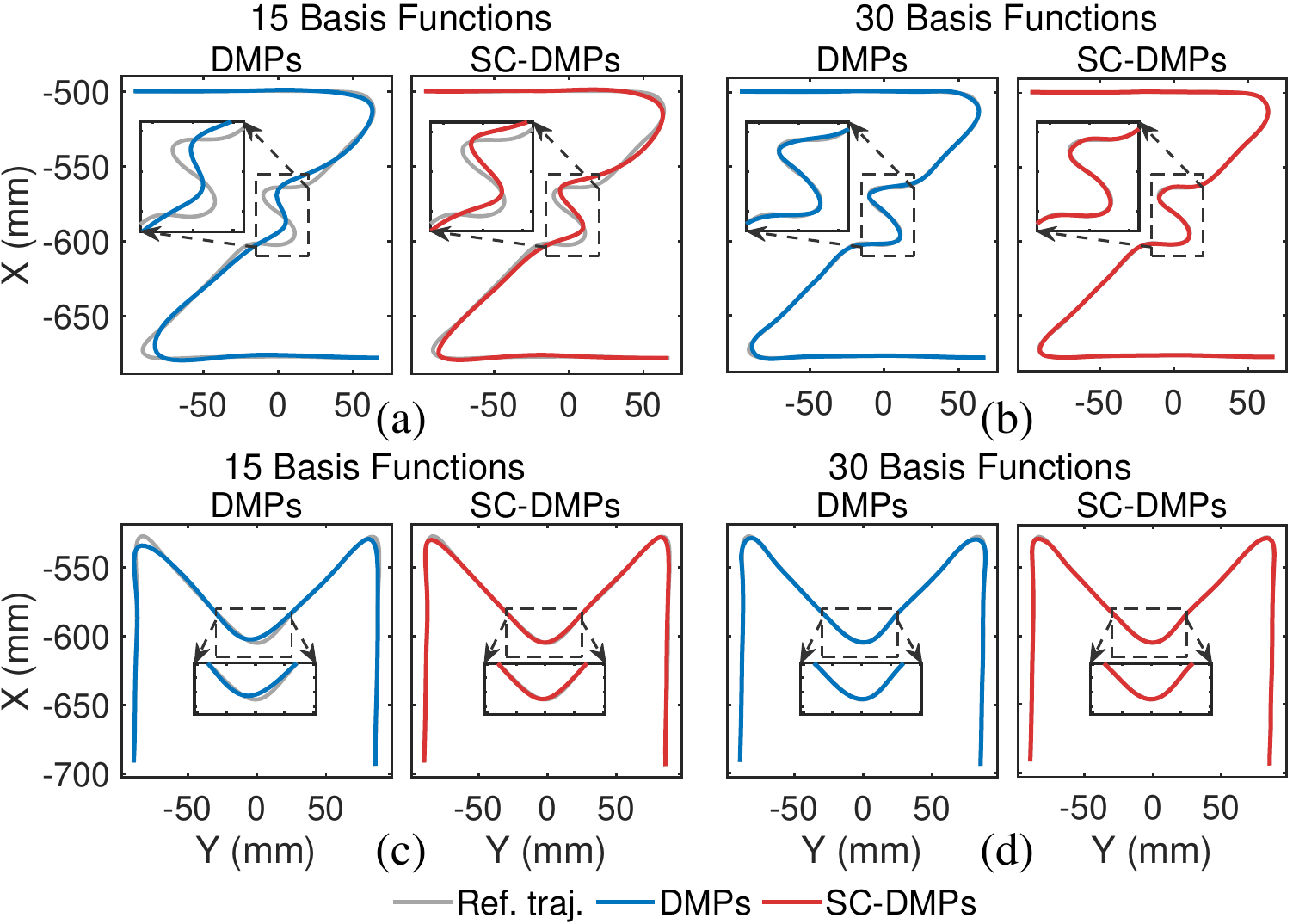}\\
	\caption{Comparison of learning performance between SC-DMPs and DMPs. (a) and (b) correspond to ``Z'' with 15 and 30 basis functions. (c) and (d) correspond to ``M'' with 15 and 30 basis functions.}
	\label{Fig_xuexi}
	\vspace{-0.6em}
\end{figure}

\subsubsection{Ablation Analysis}
This subsection presents ablation studies of the basis-allocation mechanism and stage-criticality formulation using 20 basis functions, a moderate capacity that preserves distinguishable component-level effects without excessive redundancy.

The basis-allocation components are analyzed by comparing classical DMPs, DMPs with bandwidth refinement (DMPs-BW), SC-DMPs with criticality-guided centers only (SC-DMPs-C), and full SC-DMPs, as shown in Fig.~\ref{Fig_Letter_Ablation_1}. DMPs-BW consistently reduces the reproduction errors compared with DMPs, indicating that bandwidth refinement improves local basis support even with standard centers. SC-DMPs-C also reduces the errors for both letters, confirming the contribution of stage-criticality-guided center redistribution. Full SC-DMPs further improve upon SC-DMPs-C for both letters. For ``M'', full SC-DMPs achieve the lowest RMSE, RMSE-Geo, and MaxE, demonstrating the complementary effects of center redistribution and bandwidth refinement. For ``Z'', DMPs-BW yields the lowest errors, indicating that the additional benefit of criticality-guided center redistribution depends on the geometric characteristics of the demonstrated motion.

\begin{figure}
	\centering
	\includegraphics[width=0.91\columnwidth]{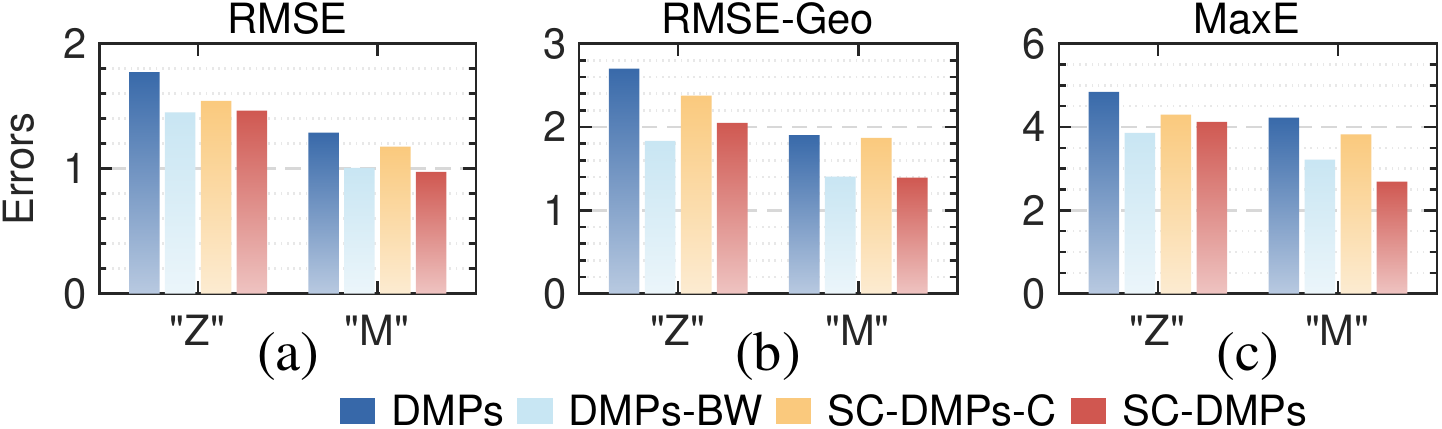}\\
	\caption{Ablation results of basis representation with 20 basis functions. (a)~RMSE (mm). (b) RMSE-Geo (mm). (c) MaxE (mm).}
	\label{Fig_Letter_Ablation_1}
	\vspace{-0.6em}
\end{figure}
\begin{figure}
	\centering
	\includegraphics[width=0.96\columnwidth]{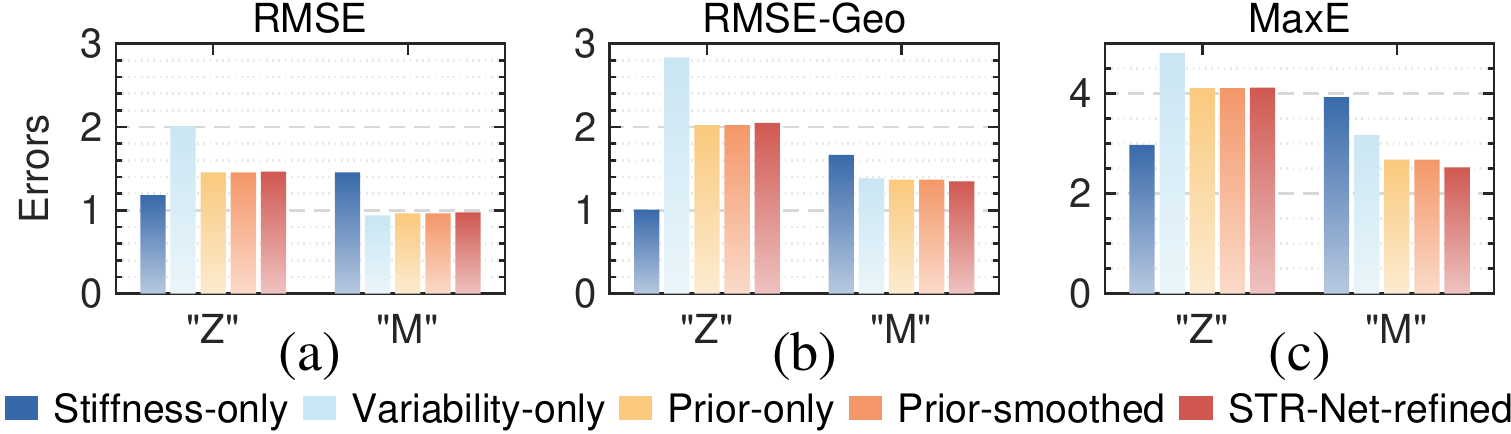}\\
	\caption{Ablation results of stage-criticality construction with 20 basis functions. (a)~RMSE (mm). (b) RMSE-Geo (mm). (c) MaxE (mm).}
	\label{Fig_Letter_Ablation_2}
	\vspace{-0.6em}
\end{figure}
\begin{table}
	\centering
	\caption{Stage-criticality refinement characteristics and construction time.}
	\label{Tab_Ablation}
	\renewcommand\arraystretch{0.90}
	\resizebox{0.95\columnwidth}{!}{
		\begin{tabular}{ p{0.6cm}<{\centering} p{1.89cm}<{\centering} p{2.16cm}<{\centering} p{1.52cm}<{\centering} p{1.3cm}<{\centering}}
			\toprule
			Letter &
			Refinement method &
			Temporal-variation reduction (\%) &
			Critical-stage retention (\%) &
			Construction time (s) \\
			\midrule
			\multirow{3}{*}{``Z''}
			& Prior-only      & 0.00  & 100.00 & $<0.001$ \\
			& Prior-smoothed  & 3.24  & 99.00  & 0.008    \\
			& STR-Net-refined & 20.91 & 91.50  & 1.258    \\
			\midrule
			\multirow{3}{*}{``M''}
			& Prior-only      & 0.00  & 100.00 & $<0.001$ \\
			& Prior-smoothed  & 3.21  & 100.00 & $<0.001$ \\
			& STR-Net-refined & 34.92 & 85.00  & 1.147    \\
			\bottomrule
		\end{tabular}
	}
	\vspace{-0.6em}
\end{table}

The construction of stage criticality is evaluated using stiffness-only, variability-only, prior-only, prior-smoothed, and STR-Net-refined variants. As shown in Fig.~\ref{Fig_Letter_Ablation_2}, stiffness-only performs better on ``Z'', whereas variability-only is more effective on ``M'', indicating that neither cue is consistently reliable across trajectories. Their fusion in the prior-only variant provides more balanced performance. Prior smoothing changes the reproduction errors only marginally, while STR-Net achieves comparable results on ``Z'' and reduces RMSE-Geo and MaxE on ``M'', with a slight increase in RMSE. Table~\ref{Tab_Ablation} further compares temporal-variation (TV) reduction, critical-stage preservation, and construction time. Prior smoothing reduces TV by about 3.2\% while retaining at least 99\% of the prior-defined critical stages. In contrast, STR-Net reduces TV by 20.91\% and 34.92\% for ``Z'' and ``M'', while retaining 91.50\% and 85.00\% of their critical stages, respectively. Thus, STR-Net provides stronger temporal refinement and partially reorganizes the stage-wise allocation priority, thereby improving reproduction accuracy for ``M'', whereas simple smoothing is sufficient for ``Z''. STR-Net requires approximately 1.2~s for offline training and less than 1~ms for inference. This one-time preprocessing cost does not affect subsequent DMP rollout or online control frequency.

\begin{figure}
	\centering
	\includegraphics[width=0.92\columnwidth]{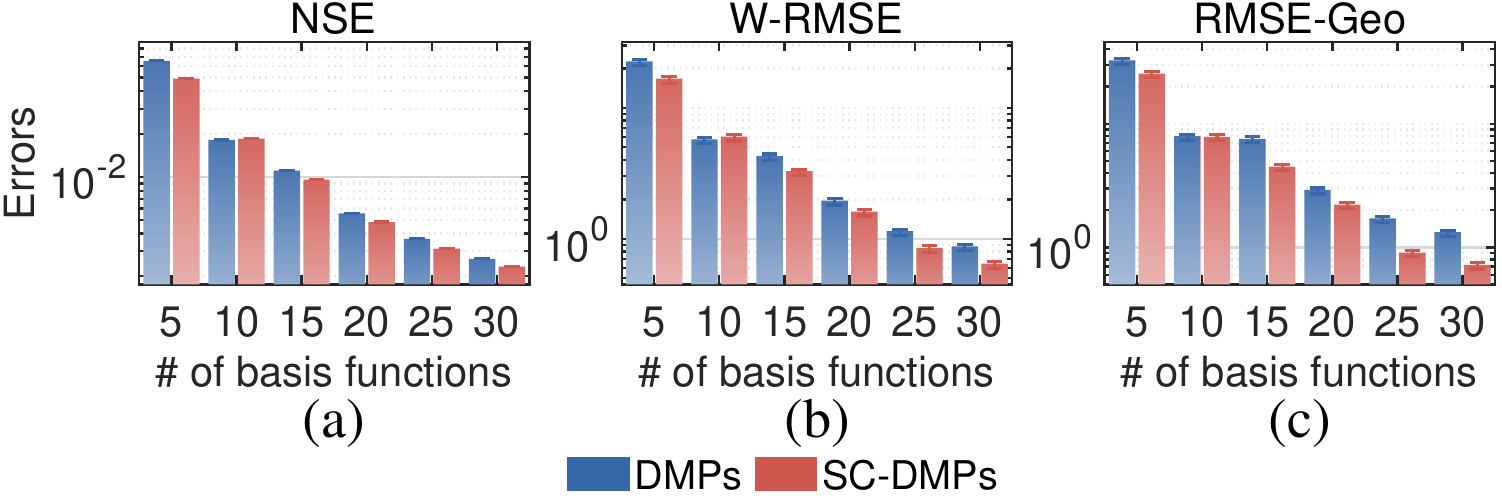}\\
	\caption{Endpoint-generalization errors for ``Z'' under varying numbers of basis functions. (a) NSE. (b) W-RMSE (mm). (c) RMSE-Geo (mm).}
	\label{Fig_Letter_Fanhau_Error_Z}
	\vspace{-0.6em}
\end{figure}
\begin{figure}
	\centering
	\includegraphics[width=0.92\columnwidth]{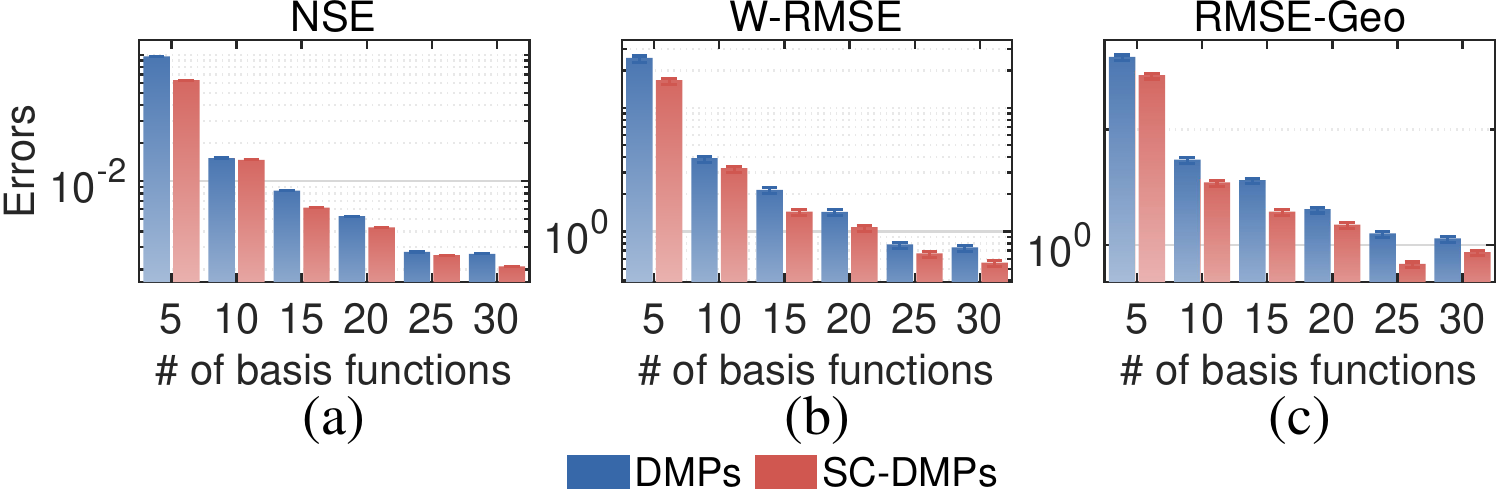}\\
	\caption{Endpoint-generalization errors for ``M'' under varying numbers of basis functions. (a) NSE. (b) W-RMSE (mm). (c) RMSE-Geo (mm).}
	\label{Fig_Letter_Fanhau_Error_M}
	\vspace{-0.6em}
\end{figure}

\subsubsection{Generalization Performance Analysis}
The generalization performance under endpoint variations and spatial scaling is further evaluated. For each letter, the model parameters are first learned and then kept fixed throughout the generalization tests, without additional demonstrations or retraining. 

Four endpoint-variation cases are evaluated, with all displacement magnitudes defined relative to the original start-goal distance: i)~\textit{Translation}, with both endpoints shifted normally by 50\%; ii)-iii)~\textit{Start-point} and \textit{Goal-point changes}, with the corresponding endpoint shifted backward or forward by 10\%; and iv)~\textit{Start-goal change}, combining a 50\% normal shift with symmetric 6\% outward endpoint displacements. Figs.~\ref{Fig_Letter_Fanhau_Error_Z} and~\ref{Fig_Letter_Fanhau_Error_M} report the mean errors and min-max ranges over these cases. SC-DMPs achieve lower average errors in 34 of the 36 comparisons; the only exceptions are small increases in NSE and W-RMSE for ``Z'' with 10 basis functions. With only 5 basis functions, the error reductions range from 21.86\% to 35.22\%. The advantage remains evident when 20 or more basis functions are used, with RMSE-Geo reductions of 23.74\%-47.32\%. Fig.~\ref{Fig_Letter_Fanhua_Z} visualizes the generalized trajectories of ``Z'' with 20 basis functions, where the enlarged local regions show that SC-DMPs more closely follow the target trajectory in locally demanding segments, consistent with their lower W-RMSE and RMSE-Geo.

To further examine spatial generalization, the models learned with 20 basis functions are kept fixed and uniformly scaled about the trajectory bounding-box center by $0.5\times$, $2\times$, $3\times$, and $5\times$. Fig.~\ref{Fig_Letter_Scaling_Z} visualizes the generalized trajectories of ``Z'', showing that SC-DMPs follow the scaled target more closely in locally demanding segments. Table~\ref{Tab_Letter_Scaling} reports the NSE, W-RMSE, and RMSE-Geo for both letters. As expected under uniform scaling, NSE remains unchanged, whereas the W-RMSE and RMSE-Geo increase proportionally with trajectory size. Nevertheless, the relative improvements of SC-DMPs are preserved across all scale factors. Compared with DMPs, SC-DMPs reduce NSE, W-RMSE, and RMSE-Geo by 12.99\%, 17.64\%, and 24.07\% for ``Z'', and by 18.74\%, 25.01\%, and 26.80\% for ``M'', respectively. These results confirm that the accuracy advantage of SC-DMPs is retained from one-half to five times the demonstrated scale.

\begin{figure}[tp]
	\centering
	\includegraphics[width=0.91\columnwidth]{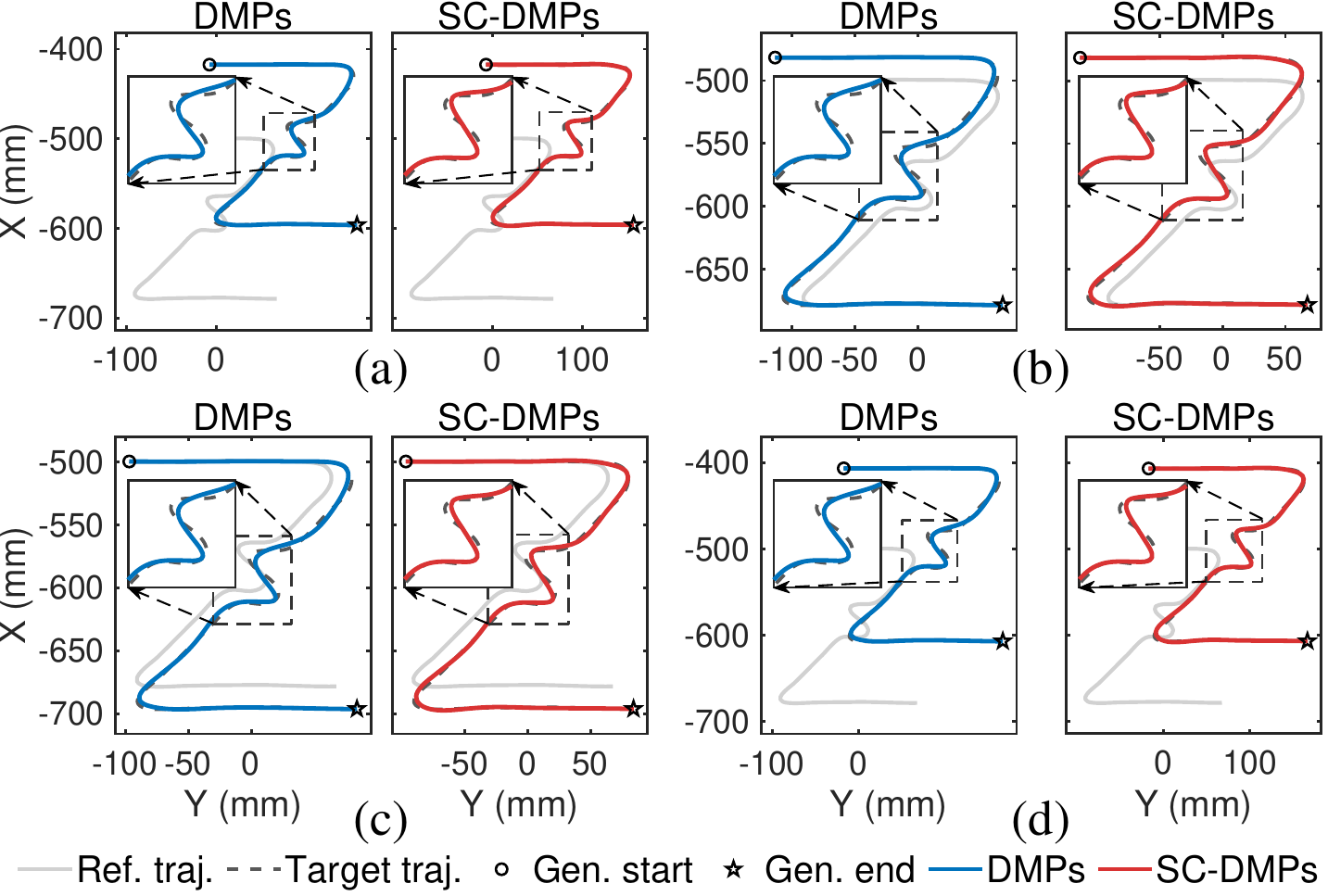}\\
	\caption{Endpoint-generalization of DMPs and SC-DMPs on ``Z'' with 20 basis functions. (a) Translation. (b) Start-point change. (c) Goal-point change. (d) Start-point and goal-point changes.}
	\label{Fig_Letter_Fanhua_Z}
	\vspace{-0.6em}
\end{figure}
\begin{figure}
	\centering
	\includegraphics[width=0.91\columnwidth]{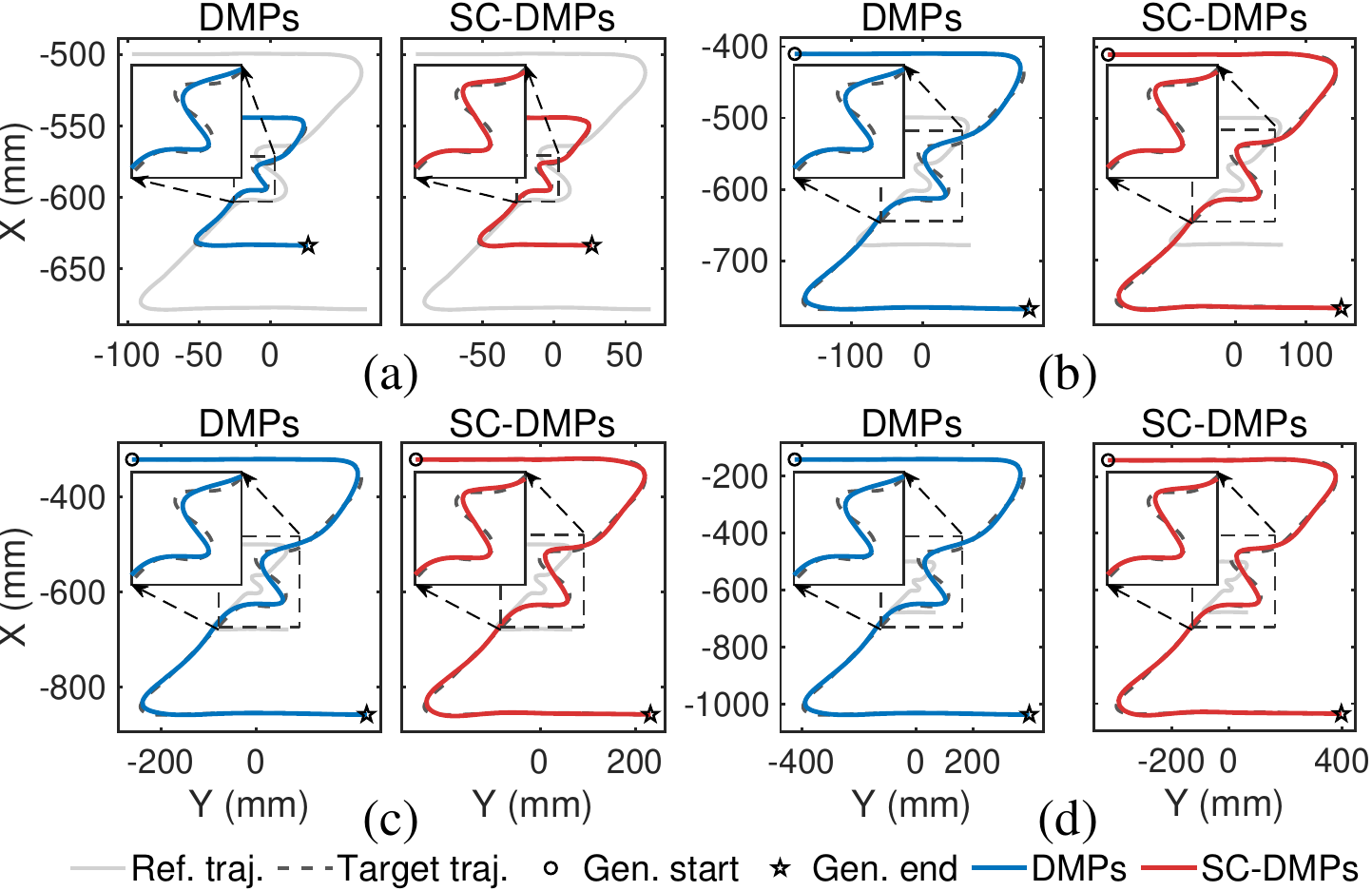}\\
	\caption{Spatial-scaling generalization of DMPs and SC-DMPs on ``Z'' with 20 basis functions. (a) 0.5$\times$. (b) 2$\times$. (c) 3$\times$. (d) 5$\times$.}
	\label{Fig_Letter_Scaling_Z}
	\vspace{-0.6em}
\end{figure}

\begin{table}[htp]
	\centering
	\caption{Errors under spatial-scaling generalization.}
	\label{Tab_Letter_Scaling}
	\renewcommand{\arraystretch}{0.90}
	\resizebox{0.95\columnwidth}{!}{
		\begin{tabular}{p{0.29cm}<{\centering}p{0.8cm}<{\centering}p{0.78cm}<{\centering}p{1.16cm}<{\centering}p{0.78cm}<{\centering}p{1.16cm}<{\centering}p{0.78cm}<{\centering}p{1.16cm}<{\centering}}
			\toprule
			\multirow{2}{*}{Letter}
			& \multirow{2}{*}{Scale}
			& \multicolumn{2}{c}{NSE}
			& \multicolumn{2}{c}{W-RMSE (mm)}
			& \multicolumn{2}{c}{RMSE-Geo (mm)} \\
			\cmidrule(lr){3-4}
			\cmidrule(lr){5-6}
			\cmidrule(lr){7-8}
			& & DMPs & SC-DMPs
			& DMPs & SC-DMPs
			& DMPs & SC-DMPs \\
			\midrule
			\multirow{4}{*}{``Z''}
			& $0.5\times$ & 0.0055 & 0.0048 & 0.9063 & 0.7465 & 1.3508 & 1.0256 \\
			& $2\times$   & 0.0055 & 0.0048 & 3.6251 & 2.9858 & 5.4032 & 4.1025 \\
			& $3\times$   & 0.0055 & 0.0048 & 5.4377 & 4.4787 & 8.1048 & 6.1537 \\
			& $5\times$   & 0.0055 & 0.0048 & 9.0628 & 7.4645 & 13.5080 & 10.2562 \\
			\cmidrule(lr){1-8}
			\multicolumn{2}{c}{Reduction}
			& \multicolumn{2}{c}{12.99\%}
			& \multicolumn{2}{c}{17.64\%}
			& \multicolumn{2}{c}{24.07\%} \\
			\midrule
			\multirow{4}{*}{``M''}
			& $0.5\times$ & 0.0053 & 0.0043 & 0.6678 & 0.5008 & 0.9511 & 0.6962 \\
			& $2\times$   & 0.0053 & 0.0043 & 2.6712 & 2.0032 & 3.8044 & 2.7847 \\
			& $3\times$   & 0.0053 & 0.0043 & 4.0068 & 3.0048 & 5.7066 & 4.1770 \\
			& $5\times$   & 0.0053 & 0.0043 & 6.6780 & 5.0080 & 9.5110 & 6.9617 \\
			\cmidrule(lr){1-8}
			\multicolumn{2}{c}{Reduction}
			& \multicolumn{2}{c}{18.74\%}
			& \multicolumn{2}{c}{25.01\%}
			& \multicolumn{2}{c}{26.80\%} \\
			\bottomrule
		\end{tabular}
	}
	\vspace{-0.6em}
\end{table}

\begin{table*}[t]
	\centering
	\caption{Comparison of learning performance, endpoint generalization, and computational efficiency on letter trajectories.}
	\label{Tab_Recent}
	\renewcommand{\arraystretch}{0.78}
	\resizebox{0.93\textwidth}{!}{%
		\begin{tabular}{p{0.7cm}<{\centering}p{2.0cm}<{\centering}p{1.9cm}<{\centering}p{2.29cm}<{\centering}p{1.96cm}<{\centering}p{2.28cm}<{\centering}p{2.3cm}<{\centering}p{2.68cm}<{\centering}}
			\toprule
			\multirow{2}{*}{Letter} & \multirow{2}{*}{Method} & \multicolumn{2}{c}{Learning performance} & \multicolumn{2}{c}{Endpoint generalization} & \multicolumn{2}{c}{Computational efficiency} \\
			\cmidrule(lr){3-4}\cmidrule(lr){5-6}\cmidrule(lr){7-8}
			& & RMSE (mm) & RMSE-Geo (mm) & NSE & RMSE-Geo (mm) & Model size (kB) & Adaptation time (ms) \\
			\midrule
			\multirow{6}{*}{``Z''} & DMPs & 0.7936 & 1.2318 & 0.0027 & 1.3303 & 1.00 & 4.77 \\
			& ProMPs & 0.6136 & 0.8083 & 0.8927 & 68.3981 & 29.08 & \textbf{0.75} \\
			& ProDMP & 0.6396 & 0.9263 & 0.1004 & 32.1902 & \textbf{0.99} & 1.66 \\
			& GP-MP & 1.1117 & 1.3649 & 0.2599 & 68.2537 & 9.65 & 89.86 \\
			& KMP & 20.6595 & 16.8022 & 0.1696 & 63.0821 & 65.64 & 938.29 \\
			& \textbf{SC-DMPs} & \textbf{0.5927} & \textbf{0.6637} & \textbf{0.0023} & \textbf{0.7168} & 1.00 & 4.71 \\
			\midrule
			\multirow{6}{*}{``M''} & DMPs & 0.6622 & 1.0638 & 0.0027 & 1.1489 & 1.00 & 4.81 \\
			& ProMPs & 0.5478 & 0.8846 & 0.5002 & 55.3787 & 29.08 & \textbf{0.75} \\
			& ProDMP & 0.5124 & 0.8321 & 0.1362 & 39.1928 & \textbf{0.99} & 1.56 \\
			& GP-MP & 1.0058 & 1.1008 & 0.2885 & 56.1967 & 9.65 & 82.52 \\
			& KMP & 12.0098 & 6.5052 & 0.1559 & 42.1927 & 65.64 & 965.69 \\
			& \textbf{SC-DMPs} & \textbf{0.5041} & \textbf{0.8094} & \textbf{0.0021} & \textbf{0.8741} & 1.00 & 4.86 \\
			\bottomrule
		\end{tabular}%
	}
	\vspace{-0.8em}
\end{table*}

\subsubsection{Comparison with Recent Methods}
Classical DMPs and four representative movement-primitive methods, including ProMPs~\cite{paraschos2013probabilistic}, ProDMP~\cite{li2023prodmp}, GP-MP~\cite{jin2023gaussian}, and KMP~\cite{huang2019kernelized}, are compared with SC-DMPs. All four baselines are implemented using their official codebases. For a controlled comparison among the explicit-basis methods, DMPs, ProMPs, ProDMP, and SC-DMPs all use 30 basis functions, while the remaining settings of ProMPs and ProDMP are kept at their defaults. Since GP-MP and KMP are kernel-based methods without explicit basis functions, they retain their native default configurations. The comparison evaluates learning performance, endpoint generalization without retraining, and computational efficiency. Table~\ref{Tab_Recent} reports RMSE and RMSE-Geo for trajectory reproduction, together with the mean NSE and RMSE-Geo over the four endpoint-generalization cases defined above. The online model size and endpoint-adaptation time are also reported. Model size counts the double-precision parameters and support states required after training, excluding original demonstrations, recomputable matrices, and offline-only variables. Adaptation time covers endpoint updating and trajectory generation, excluding training and file loading.

As shown in Table~\ref{Tab_Recent}, SC-DMPs achieve the lowest learning and endpoint-generalization errors for both letters. Relative to the best competing result in each column, SC-DMPs reduce the learning RMSE by 3.4\% and 1.6\%, the learning RMSE-Geo by 17.9\% and 2.7\%, the endpoint-generalization NSE by 11.9\% and 20.4\%, and the endpoint-generalization RMSE-Geo by 46.1\% and 23.9\% for ``Z'' and ``M'', respectively. These results indicate that criticality-guided allocation improves both trajectory reproduction and geometric fidelity under endpoint changes. SC-DMPs require only approximately 1~kB of online model storage, comparable to DMPs and ProDMP and substantially smaller than ProMPs, GP-MP, and KMP. Although ProMPs and ProDMP achieve shorter adaptation times, SC-DMPs remain close to classical DMPs and are considerably faster than GP-MP and KMP. Overall, SC-DMPs achieve a favorable balance between accuracy and model complexity among the evaluated methods.

\begin{table}[t]
	\centering
	\caption{Tested hyperparameter settings and default values.}
	\label{Tab_Parameter_Settings}
	\renewcommand{\arraystretch}{0.82}
	\resizebox{0.90\columnwidth}{!}{
		\begin{tabular}{l l p{1.35cm}<{\centering}}
			\toprule
			Parameter & Tested settings & Default \\
			\midrule
			$\beta_1/\beta_2$
			&
			\begin{tabular}[c]{@{}l@{}}
				0.50/1.50, 0.75/1.25, 1.00/1.00,\\
				1.25/0.75, 1.50/0.50
			\end{tabular}
			& 1.00/1.00 \\
			\midrule
			GMM components
			& 14, 18, 22, 26, 30
			& 22 \\
			\midrule
			Latent dimension (LD)
			& 32, 64, 128
			& 64 \\
			\midrule
			Kernel size (KS)
			& 3, 5, 7
			& 5 \\
			\midrule
			Bandwidth candidate (BW)
			&
			\begin{tabular}[c]{@{}l@{}}
				Narrow: \{0.75{:}0.05{:}1.25\}\\
				Normal: \{0.50{:}0.10{:}1.50\}\\
				Wide: \{0.25{:}0.15{:}1.75\}
			\end{tabular}
			& Normal \\
			\bottomrule
		\end{tabular}
	}
	\vspace{-0.6em}
\end{table}

\subsubsection{Sensitivity and Stability Analysis}
To assess the robustness of SC-DMPs to parameter variations, sensitivity and stability analyses are conducted. Table~\ref{Tab_Parameter_Settings} summarizes the tested settings together with the default values adopted in the preceding experiments. Fig.~\ref{Fig_Letter_Parameter_Sensitivity_Boxplot} reports the RMSE-Geo reduction relative to DMPs, where the boxes summarize the results over the candidate settings and the diamonds indicate the defaults. All tested settings retain positive error reductions for both letters, demonstrating that the advantage of SC-DMPs is not restricted to a particular parameter combination. The kernel-size settings exhibit the smallest variations, with standard deviations of only 0.22 and 0.25 percentage points for ``Z'' and ``M''. Greater variations are observed for the stiffness-variability ratio and the number of GMM components, but their reductions remain positive. The default settings achieve reductions of 45.64\% and 24.32\% for the two letters and provide a consistent configuration rather than a letter-specific optimum.

\begin{figure}
	\centering
	\includegraphics[width=0.92\columnwidth]{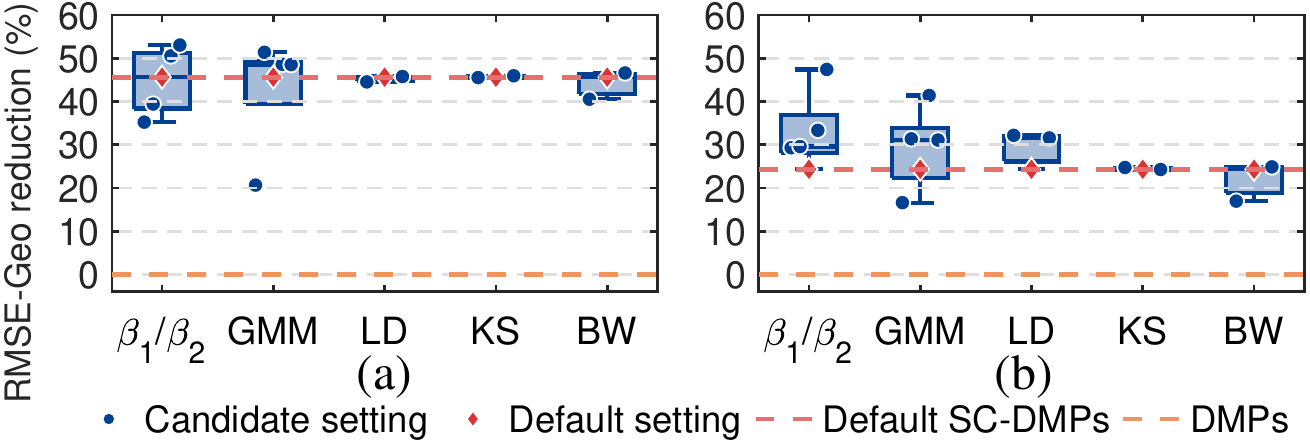}\\
	\caption{Hyperparameter sensitivity of SC-DMPs. (a) ``Z''. (b) ``M''.}
	\label{Fig_Letter_Parameter_Sensitivity_Boxplot}
	\vspace{-0.7em}
\end{figure}
\begin{figure}[t]
	\centering
	\includegraphics[width=0.89\columnwidth]{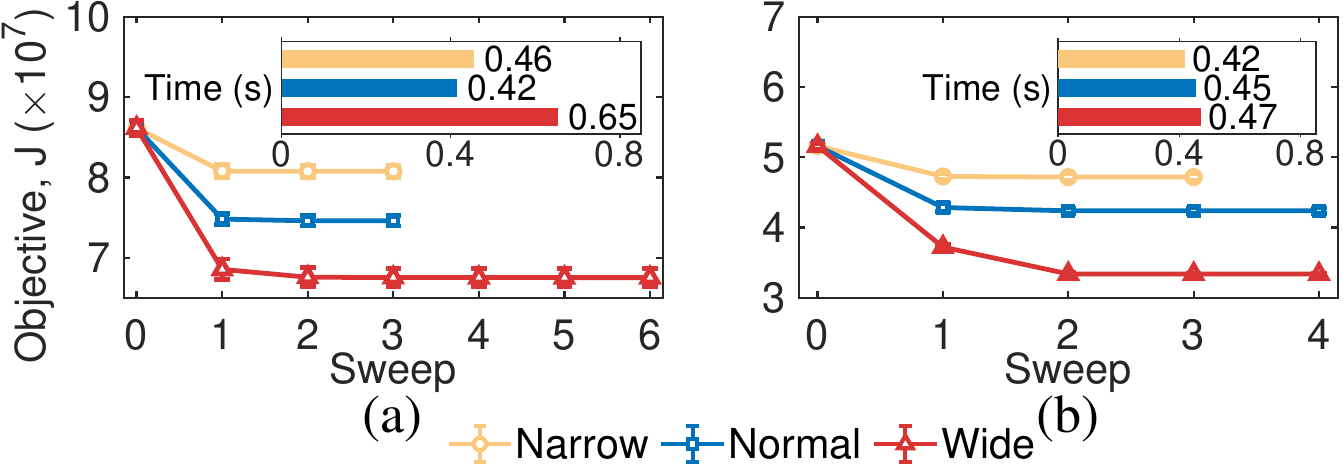}\\
	\caption{Convergence and computational time of bandwidth optimization. (a) ``Z''. (b) ``M''.}
	\label{Fig_Bandwidth_Convergence}
	\vspace{-0.6em}
\end{figure}
\begin{figure}
	\centering
	\includegraphics[width=0.935\columnwidth]{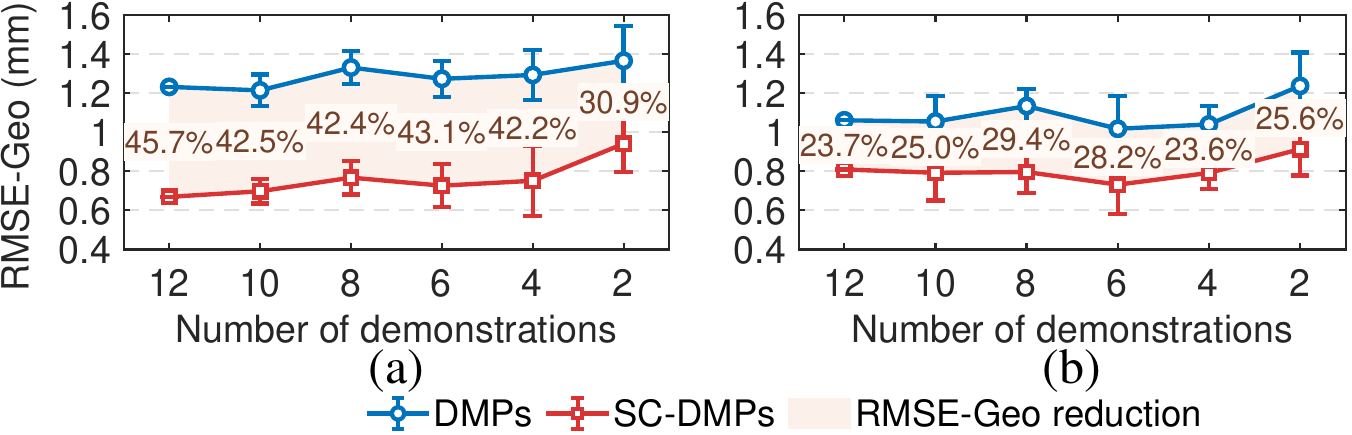}\\
	\caption{Effect of demonstration number on the RMSE-Geo of DMPs and SC-DMPs. (a) ``Z''. (b) ``M''.}
	\label{Fig_Demonstration_Number}
	\vspace{-0.7em}
\end{figure}

The stability of bandwidth optimization is evaluated using the narrow, normal, and wide candidate sets. A coordinate update is accepted only when it decreases the forcing-term residual objective by more than $10^{-8}$, with a maximum of 10 sweeps. Since the candidate space is finite and every accepted update strictly decreases the objective, the procedure terminates at a coordinate-wise solution over the discrete candidate set, although global optimality is not guaranteed. As shown in Fig.~\ref{Fig_Bandwidth_Convergence}, all 18 runs exhibit monotonic objective reduction and terminate before reaching the maximum-sweep limit. Convergence requires 3-6 sweeps and approximately 0.42-0.65~s, while reducing the objective by 6.26\%-35.24\%. These results verify stable finite termination under the tested configurations with a small offline computational cost.

The effect of demonstration number is evaluated using 2, 4, 6, 8, 10, and 12 demonstrations. For each setting from 2 to 10, ten fixed random subsets are tested, with the GMR reference, stage criticality, and movement-primitive models reconstructed independently; the complete dataset is used once for the 12-demonstration case. Fig.~\ref{Fig_Demonstration_Number} reports the mean RMSE-Geo and standard deviation, while the shaded regions and annotations indicate the mean paired reductions achieved by SC-DMPs. SC-DMPs consistently obtain lower errors across all tested demonstration numbers. Even with only two demonstrations, the reductions remain 30.86\%$\pm$8.87\% for ``Z'' and 25.62\%$\pm$9.75\% for ``M''. With 4-12 demonstrations, the mean reductions remain within 42.16\%-45.70\% for ``Z'' and 23.60\%-29.36\% for ``M''. Although different subsets cause some error fluctuations, SC-DMPs consistently outperform DMPs across all tested demonstration numbers.

\begin{figure*}
	\centering
	\includegraphics[width=0.91\textwidth]{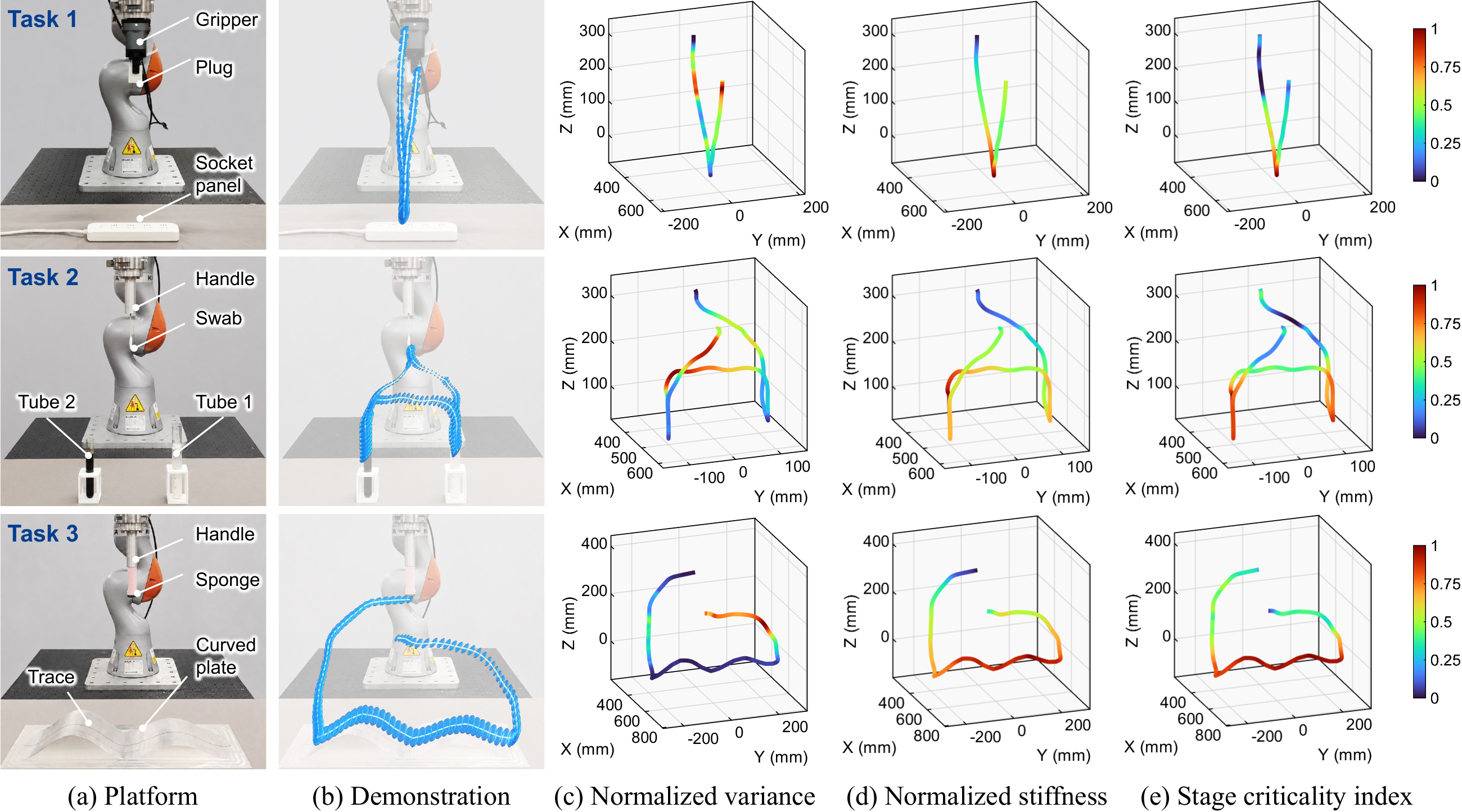}\\
	\caption{Experimental setups and data visualizations for plug insertion and removal, sequential dipping, and curved-surface wiping. (a) Platform. (b) Demonstrated trajectory and stiffness ellipsoids. (c) Normalized variance. (d) Normalized stiffness. (e) Stage criticality index.}
	\label{Fig_Real_Robot}
	\vspace{-0.7em}
\end{figure*}

\begin{figure}
	\centering
	\includegraphics[width=0.78\columnwidth]{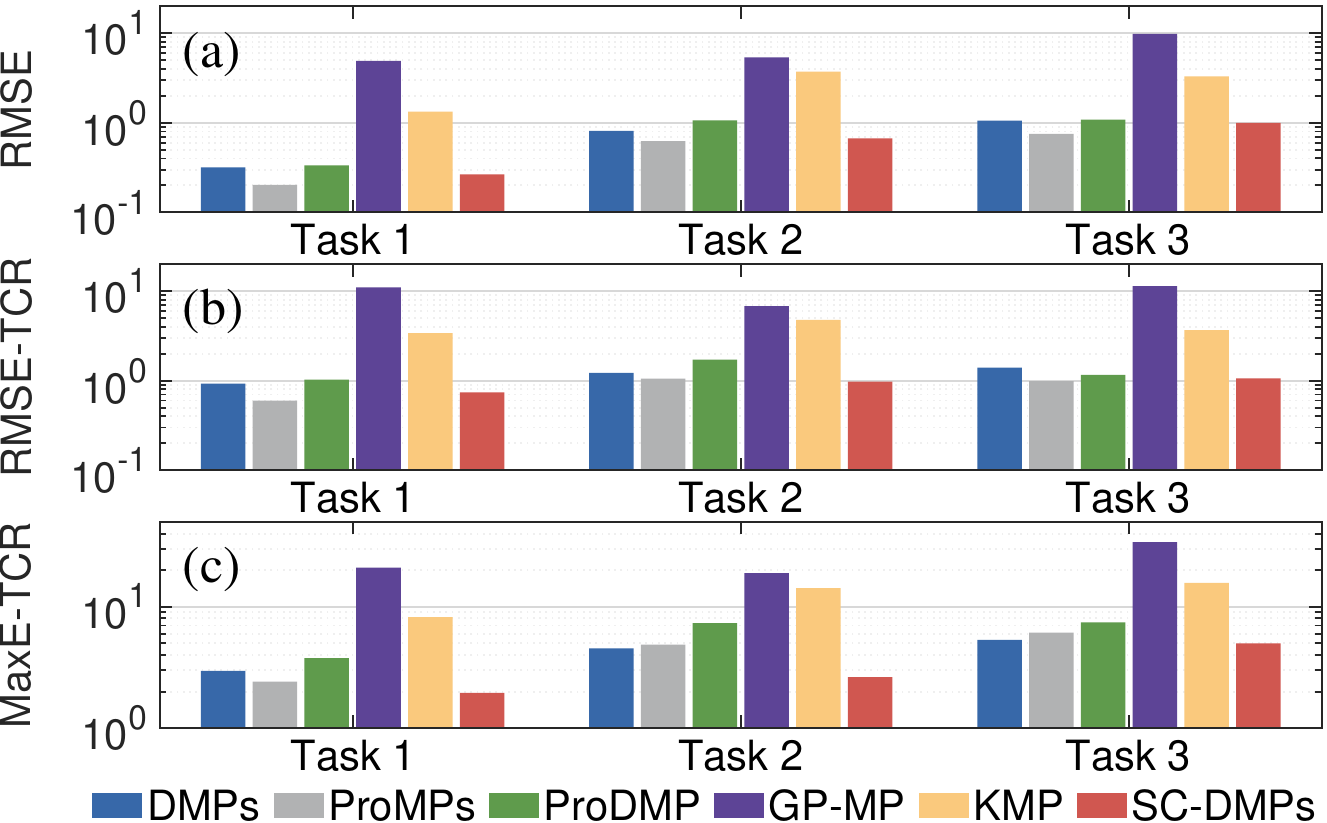}\\
	\caption{Trajectory-learning error comparison on three robotic tasks. (a)~RMSE (mm). (b) RMSE-TCR (mm). (c) MaxE-TCR (mm).}
	\label{Fig_Robot_Main_Accuracy}
	\vspace{-0.7em}
\end{figure}

\subsection{Real-Robot Experiments}
To evaluate the proposed method under different geometric and contact constraints, three real-robot tasks are considered, as shown in Fig.~\ref{Fig_Real_Robot}: i) plug insertion and removal under tight positional constraints; ii) sequential dipping involving two tube insertion-withdrawal stages; and iii) curved-surface wiping under continuous contact. Each task is demonstrated four times via kinesthetic teaching, with the Cartesian trajectory and estimated interaction stiffness synchronously recorded. After uniform resampling to 1000 points, the demonstrations are used to construct the GMR reference trajectory. The variance and stiffness profiles shown in Fig.~\ref{Fig_Real_Robot}(c) and (d) are independently normalized and fused to obtain the stage-criticality index in Fig.~\ref{Fig_Real_Robot}(e). High criticality is concentrated around the plug-socket interaction, tube insertion-withdrawal stages, and surface-contact portion of the wiping trajectory, where geometric and contact constraints impose stricter local reproduction requirements. In contrast, less constrained transfer segments receive lower criticality. The resulting profiles therefore provide a physically interpretable correspondence between the demonstrated interaction characteristics and task-dependent accuracy requirements.

\begin{table}[t]
	\centering
	\caption{STR-Net ablation on robot tasks. All errors are in mm}
	\label{Tab_Robot_STRNet_Ablation}
	\renewcommand{\arraystretch}{0.9}
	\resizebox{0.94\columnwidth}{!}{%
		\begin{tabular}{ccp{0.9cm}<{\centering}cc}
			\toprule
			Task & Refinement & RMSE & RMSE-TCR & MaxE-TCR \\
			\midrule
			\multirow{3}{*}{Task 1} & Prior-only & 0.2709 & 0.7714 & 2.0715 \\
			& Prior-smoothed & 0.2709 & 0.7716 & 2.0719 \\
			& STR-Net-refined & \textbf{0.2666} & \textbf{0.7409} & \textbf{1.9526} \\
			\midrule
			\multirow{3}{*}{Task 2} & Prior-only & 0.6737 & 0.9580 & 2.5812 \\
			& Prior-smoothed & 0.6736 & \textbf{0.9578} & \textbf{2.5807} \\
			& STR-Net-refined & \textbf{0.6709} & 0.9743 & 2.6242 \\
			\midrule
			\multirow{3}{*}{Task 3} & Prior-only & 1.0183 & 1.1171 & 5.3894 \\
			& Prior-smoothed & 1.0186 & 1.1177 & 5.3926 \\
			& STR-Net-refined & \textbf{0.9982} & \textbf{1.0621} & \textbf{4.9892} \\
			\midrule
			\multicolumn{2}{c}{Average reduction vs. prior-only} & 1.38\% & 2.43\% & 4.74\% \\
			\multicolumn{2}{c}{Average reduction vs. prior-smoothed} & 1.40\% & 2.45\% & 4.77\% \\
			\bottomrule
		\end{tabular}%
	}
	\vspace{-0.6em}
\end{table}

For the trajectory-learning comparison, all methods learn the same GMR reference. DMPs, ProMPs, ProDMP, and SC-DMPs use 50 basis functions, whereas GP-MP and KMP retain their native default configurations. RMSE evaluates the overall trajectory error, while RMSE-TCR and MaxE-TCR measure the root-mean-square and worst-case errors within the task-critical regions, respectively. These regions cover the plug insertion-removal stages in Task~1, the tube insertion-withdrawal stages in Task~2, and the surface-contact stage in Task~3. As shown in Fig.~\ref{Fig_Robot_Main_Accuracy}, SC-DMPs produce lower RMSE and RMSE-TCR than classical DMPs on all three tasks. The corresponding reductions range from 5.83\% to 17.51\% for RMSE and from 19.93\% to 23.86\% for RMSE-TCR. SC-DMPs also achieve the lowest MaxE-TCR among all six methods, with reductions of 6.56\%-42.01\% relative to DMPs. Although ProMPs attains lower mean trajectory errors in several comparisons, its MaxE-TCR remains higher in every task. These results highlight the suppression of large deviations in task-critical stages.

To further assess the role of STR-Net under real-robot geometric and contact constraints, Table~\ref{Tab_Robot_STRNet_Ablation} compares the stiffness-variability prior (prior-only), its moving-average refinement (prior-smoothed), and the STR-Net-refined variant, with all other settings unchanged. Prior smoothing produces nearly unchanged errors and no consistent improvement. Relative to prior-only, STR-Net reduces the across-task mean RMSE, RMSE-TCR, and MaxE-TCR by 1.38\%, 2.43\%, and 4.74\%, respectively, with similar gains over prior smoothing. It achieves the lowest errors in Tasks~1 and~3, while slightly lowering RMSE but producing marginally higher critical-region errors in Task~2. These results suggest that STR-Net can provide additional temporal refinement over the tested smoothing baseline.

\section{Conclusion}\label{Conclusion}
This article proposes SC-DMPs, which integrate operator-robot interaction stiffness and cross-demonstration trajectory variability into a stage-criticality index for adaptive basis allocation. Guided by this index, basis centers are redistributed through inverse cumulative criticality, while bandwidth refinement adjusts local support to provide denser representation at high-criticality stages. Experiments on handwriting trajectories and three real-robot tasks demonstrate improved trajectory reproduction, generalization, and task-critical accuracy while retaining the compact and stable DMP formulation.

The current formulation targets precision-oriented skills whose critical stages are reflected in interaction stiffness and trajectory consistency. Extending it to low-stiffness critical stages associated with compliant behavior, task dynamics, or environmental context may require complementary force, vision, object-dynamics, and contextual cues.

\bibliographystyle{IEEEtran}
\bibliography{IEEEabrv,mybib}

\end{document}